\documentclass{article} 
\PassOptionsToPackage{numbers, compress}{natbib}
\usepackage[preprint]{neurips_2025}

\usepackage{amsmath,amsfonts,bm}

\def\eqref#1{equation~\ref{#1}}

\def\1{\bm{1}}

\DeclareMathAlphabet{\mathsfit}{\encodingdefault}{\sfdefault}{m}{sl}
\SetMathAlphabet{\mathsfit}{bold}{\encodingdefault}{\sfdefault}{bx}{n}

\newcommand{\E}{\mathbb{E}}

\newcommand{\R}{\mathbb{R}}

\usepackage{multirow}
\usepackage[english]{babel}
\usepackage{amssymb,amsfonts}
\usepackage{float}
\usepackage{booktabs}
\usepackage{color}
\usepackage{listings}
\usepackage{rotating}
\usepackage{pifont}
\usepackage{stmaryrd}
\usepackage{array}
\usepackage{subcaption}
\usepackage{enumitem}

\usepackage{graphicx}
\usepackage{multicol}
\usepackage{framed}
\usepackage[
	operators,
	sets,
	landau,
	probability,
	notions,
	logic,
	ff,
	mm,
	advantage,
	events,
	complexity,
	asymptotics,
	keys,
	lambda]{cryptocode}
\usepackage{adjustbox}
\usepackage{varioref}
\usepackage{hyperref}
\usepackage{breakcites}
\usepackage{placeins}
\usepackage{orcidlink}
\usepackage{tcolorbox}

\usepackage{url} %
\usepackage{framed}
\usepackage{comment}
\usepackage[nameinlink,capitalize]{cleveref}

\usepackage{pgf, tikz}
\usetikzlibrary{arrows, automata, positioning}

\floatstyle{boxed}
\usepackage{xcolor}
\definecolor{cadmiumorange}{rgb}{0.93, 0.53, 0.18}
\definecolor{emerald}{rgb}{0.31, 0.78, 0.47}
\definecolor{amaranth}{rgb}{0.9, 0.17, 0.31}
\definecolor{candypink}{rgb}{0.89, 0.44, 0.48}
\definecolor{caribbeangreen}{rgb}{0.0, 0.8, 0.6}
\definecolor{cornflowerblue}{rgb}{0.39, 0.58, 0.93}
\definecolor{limegreen}{rgb}{0.2, 0.8, 0.2}
\definecolor{mayablue}{rgb}{0.21,0.49,0.74}

\usepackage{utfsym}

\definecolor{ThemeOrange}{HTML}{FFA532}     
\definecolor{ThemeGoldBrown}{HTML}{A98426}  
\definecolor{ThemeNavy}{HTML}{0E2344}       
\definecolor{ThemeBlue}{HTML}{0073CA}       
\definecolor{ThemePurple}{HTML}{8A05B7}    
\definecolor{ThemeGray}{HTML}{7f7f7f}       

\newcommand{\introtakeawaybox}[2][]{%
  \begin{tcolorbox}[%
    colback=ThemeOrange!10,
    colframe=ThemeGoldBrown!90,
    boxrule=0.5pt,
    arc=0mm,
    left=5pt,
    right=5pt,
    top=2pt,
    bottom=2pt,
    fontupper={\fontsize{9.5pt}{11.75pt}\selectfont}
  ]
    \IfNoValueTF{#1}{\textbf{Contributions:} #2}{\textbf{Contributions #1:} #2}%
  \end{tcolorbox}%
  \vspace*{-0.1cm}
}

\newcommand{\empiricaltakeawaybox}[2][]{%
  \begin{tcolorbox}[
    colback=ThemeOrange!10, %
    colframe=ThemeGoldBrown!90,
    float=t,
    boxrule=0.5pt,
    arc=0mm,
    left=5pt,
    right=5pt,
    top=2pt,
    bottom=2pt,
    fontupper={\fontsize{9.5pt}{11.75pt}\selectfont}
  ]
    \IfNoValueTF{#1}{\textbf{Empirical Takeaway:} #2}{\textbf{Empirical Takeaways #1:} #2} %
  \end{tcolorbox}%
  \vspace*{-0.1cm}
}

\newcommand{\takeawaybox}[2][]{%
  \begin{tcolorbox}[
    colback=ThemeOrange!10, %
    colframe=ThemeGoldBrown!90,
    boxrule=0.5pt,
    arc=0mm,
    left=5pt,
    right=5pt,
    top=2pt,
    bottom=2pt,
    fontupper={\fontsize{9.5pt}{11.75pt}\selectfont}
  ]
    \IfNoValueTF{#1}{\textbf{Takeaway:} #2}{\textbf{#1} #2} %
  \end{tcolorbox}%
  \vspace*{-0.1cm}
}
\hypersetup{
  linkcolor   = ThemeNavy,
  urlcolor    = ThemeBlue,
  citecolor   = ThemeBlue,
  filecolor   = ThemeGray,
  colorlinks  = true,
  breaklinks  = true,
  bookmarksopen = true
}

\usepackage{epigraph}

\newcommand{\ignore}[1]{}

\newcommand{\eqdef}{\stackrel{\mathrm{def}}{=}}

\def\<{\left\langle}
\def\>{\right\rangle}
\def\[{\left[}
\def\]{\right]}
\def\({\left(}
\def\){\right)}

\newcommand{\clip}{\textbf{clip}}

\usepackage{amsthm}
\newtheorem{thm}{Theorem}

\newtheorem{asp}{Assumption}
\newtheorem{lemma}[thm]{Lemma}

\usepackage{amsmath,amsfonts,bm}

\def\eqref#1{equation~\ref{#1}}

\def\1{\bm{1}}

\DeclareMathAlphabet{\mathsfit}{\encodingdefault}{\sfdefault}{m}{sl}
\SetMathAlphabet{\mathsfit}{bold}{\encodingdefault}{\sfdefault}{bx}{n}

\renewcommand{\eqref}[1]{(\ref{#1})}

\usepackage{hyperref}
\usepackage{url}

\usepackage[hang,flushmargin]{footmisc}
\usepackage{pgfplotstable}
\usepackage{xspace}
\usepackage{paracol}
\usepackage{bibunits}
\usepackage[toc, title]{appendix}
\usepackage{minitoc}

\usepackage[colorinlistoftodos,bordercolor=orange,backgroundcolor=orange!20,linecolor=orange,textsize=scriptsize]{todonotes}

\usepackage[utf8]{inputenc} 
\usepackage[T1]{fontenc}    
\usepackage{hyperref}       
\usepackage{url}            
\usepackage{booktabs}       
\usepackage{amsfonts}       
\usepackage{nicefrac}       
\usepackage{microtype}      
\definecolor{grey_plot}{HTML}{7f7f7f}
\definecolor{red_plot}{HTML}{EB2644}
\usepackage{algorithm}
\usepackage[noend]{algpseudocode}
\usepackage{setspace}

\algnewcommand{\IfThenElse}[3]{
  \algorithmicif\ #1\ \algorithmicthen\ #2\ \algorithmicelse\ #3}

\newcommand{\method}{\texttt{MT-DAO}\xspace}
\newcommand{\methodadam}{\method-\texttt{Adam}\xspace}
\newcommand{\methodadopt}{\method-\texttt{ADOPT}\xspace}
\newcommand{\methodsgdm}{\method-\texttt{SGDM}\xspace}
\newcommand{\adam}{\texttt{Adam}\xspace}
\newcommand{\muon}{\texttt{Muon}\xspace}
\newcommand{\dion}{\texttt{Dion}\xspace}
\newcommand{\ddp}{\texttt{DDP}\xspace}

\newcommand{\localsgd}{\texttt{Local} \texttt{SGD}\xspace}

\newcommand{\qhm}{\texttt{QHM}\xspace}
\newcommand{\ema}{\texttt{EMA}\xspace}
\newcommand{\nesterov}{\texttt{Nesterov}\xspace}
\newcommand{\ademamix}{\texttt{AdEMAMix}\xspace}
\newcommand{\aggmo}{\texttt{AggMo}\xspace}
\newcommand{\desloc}{\texttt{DES}-\texttt{LOC}\xspace}
\newcommand{\deslocadam}{\texttt{DES}-\texttt{LOC}-\texttt{ADAM}\xspace}
\newcommand{\deslocadopt}{\texttt{DES}-\texttt{LOC}-\texttt{ADOPT}\xspace}

\newcommand{\localadopt}{\texttt{Local} \texttt{ADOPT}\xspace}
\newcommand{\localadam}{\texttt{Local} \texttt{Adam}\xspace}

\newcommand{\adopt}{\texttt{ADOPT}\xspace}
\newcommand{\fullmethod}{Multi-timescale Distributed Adaptive Optimizers\xspace}

\title{\method: \fullmethod with Local Updates}

\let\hypersetup\truehypersetup
\makeatletter
\newcommand{\myfnsymbol}[1]{%
  \expandafter\@myfnsymbol\csname c@#1\endcsname
}
\newcommand{\@myfnsymbol}[1]{%
  \ifcase #1
  \or 1
  \or 2
  \or 3
  \or 4
  \or 5%
\or 6%
  \or \TextOrMath{\textasteriskcentered}{*}
  \or \TextOrMath{\textasteriskcentered}{*}\TextOrMath{\textasteriskcentered}{*}
  \or \TextOrMath{\textdagger}{\dagger}
  \or \TextOrMath{\textasteriskcentered}{*},\TextOrMath{\textasteriskcentered}{*}\TextOrMath{\textasteriskcentered}{*}
  \fi
}
\newcommand{\affiliationA}{\@myfnsymbol{1}}
\newcommand{\affiliationB}{\@myfnsymbol{2}}
\newcommand{\affiliationC}{\@myfnsymbol{3}}
\newcommand{\affiliationD}{\@myfnsymbol{4}}
\newcommand{\affiliationE}{\@myfnsymbol{5}}
\newcommand{\affiliationF}{\@myfnsymbol{6}}
\newcommand{\equalcontributor}{\@myfnsymbol{7}}
\newcommand{\biequalcontributor}{\@myfnsymbol{8}}
\newcommand{\correspondingA}{\@myfnsymbol{9}}

\author{
Alex Iacob\textsuperscript{\correspondingA,\affiliationA,\affiliationB}
\And  
Andrej Jovanović\textsuperscript{\equalcontributor,\affiliationA}
\And Mher Safaryan\textsuperscript{\equalcontributor,\affiliationC}
\And  
Meghdad Kurmanji\textsuperscript{\affiliationA}
\And  
Lorenzo Sani\textsuperscript{\affiliationA,\affiliationB}
\And  
Samuel Horváth\textsuperscript{\affiliationD}
\And 
William F. Shen\textsuperscript{\affiliationA}
\And 
Xinchi Qiu\textsuperscript{\affiliationA}
\And 
Nicholas D. Lane\textsuperscript{\affiliationA,\affiliationB}
}
\makeatother
\addtocontents{toc}{\protect\setcounter{tocdepth}{2}} 

\begin{document}
\doparttoc
\faketableofcontents

{
\begingroup
\begin{figure}[t]
\vspace{-2.25cm}
    \quad
    \begin{subfigure}{0.1275\textwidth}
        \includegraphics[width=\textwidth]{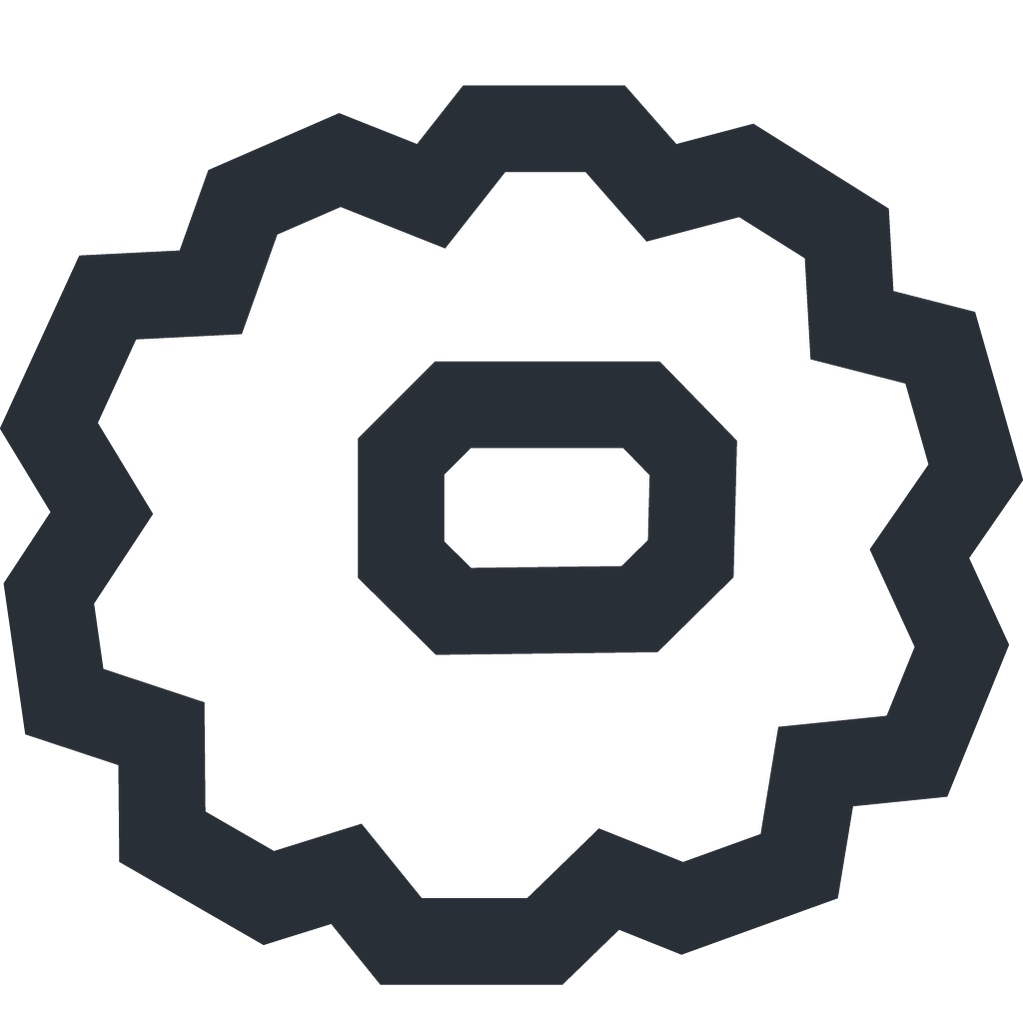}
    \end{subfigure}
    \hfill
    \begin{subfigure}{0.1\textwidth}
        \includegraphics[width=\textwidth]{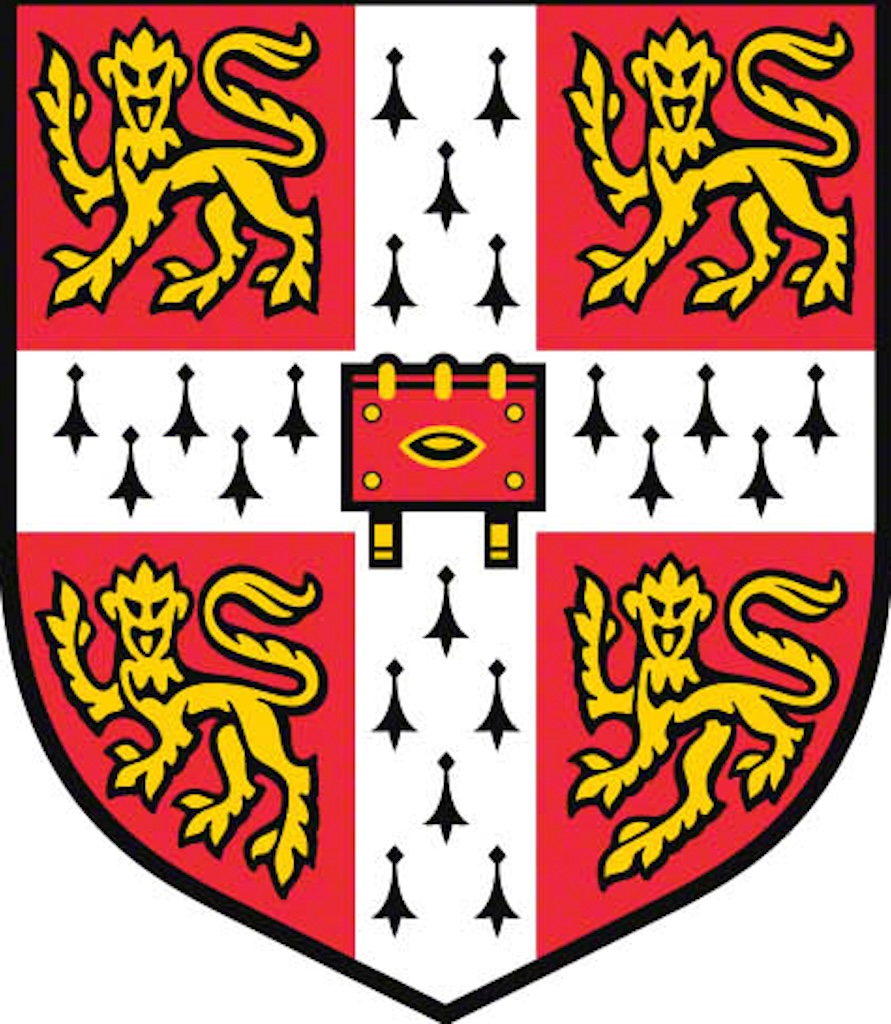}
    \end{subfigure}
    \hfill
    \begin{subfigure}{0.1275\textwidth}
        \includegraphics[width=\textwidth]{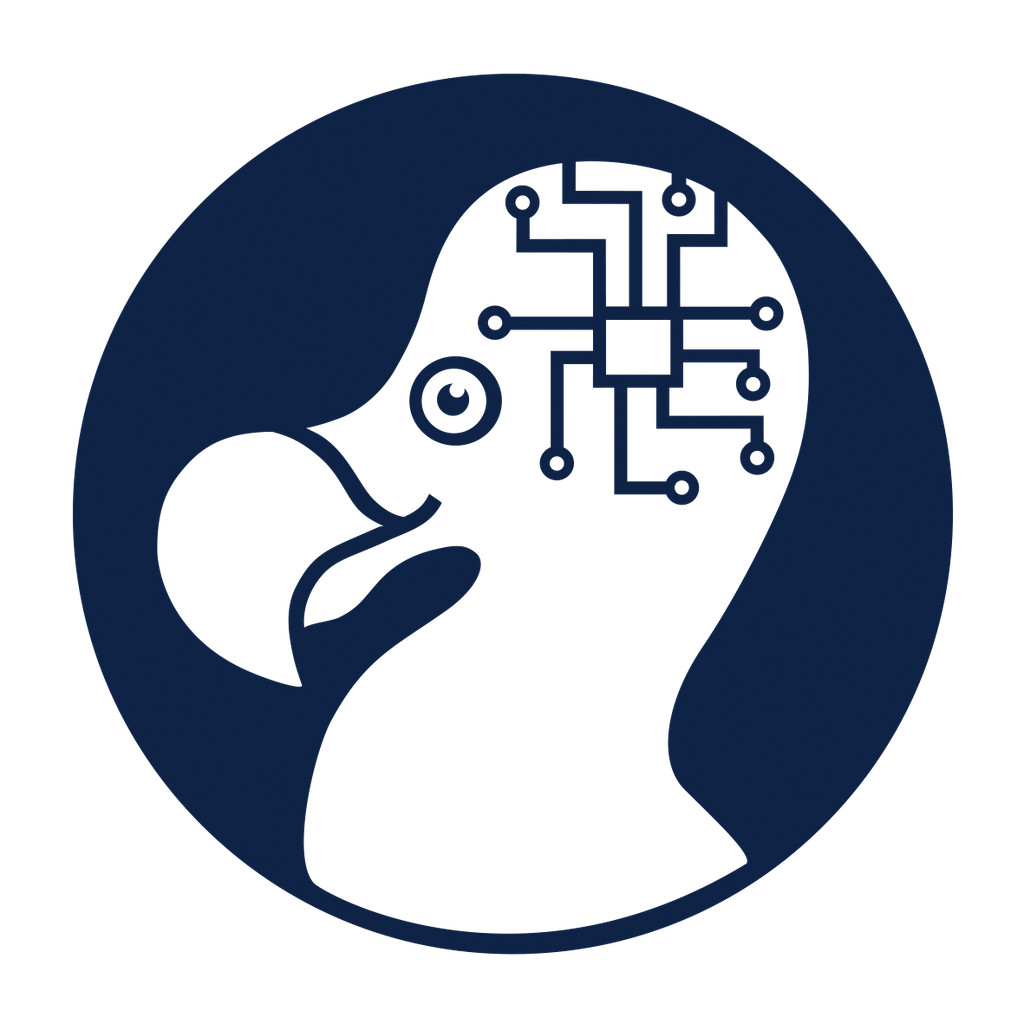}
    \end{subfigure}
    \vspace{-0.75cm}
\end{figure}

\endgroup
}
\maketitle

\renewcommand{\thefootnote}{\myfnsymbol{footnote}}
\footnotetext{\textsuperscript{\textdagger}\href{mailto:aai30@cam.ac.uk}{\nolinkurl{aai30@cam.ac.uk}}; \textsuperscript{*}Equal contributions; \textsuperscript{1}University of Cambridge; \textsuperscript{2}Flower Labs; \textsuperscript{3}Institute of Science and Technology Austria;  \textsuperscript{4}Mohamed bin Zayed University of Artificial Intelligence}

\begin{abstract}

Training large models with distributed data parallelism (\ddp) requires frequent communication of gradients across workers, which can saturate bandwidth. \emph{Infrequent} communication strategies (e.g., Local SGD) reduce this overhead but, when applied to adaptive optimizers, often suffer a performance gap relative to \emph{fully synchronous} \ddp.
We trace this gap to a time-scale mismatch: the optimizer's fast-moving momentum, tuned for frequent updates, decays too quickly to smooth gradients over long intervals, leading to noise-dominated optimization. To address this, we propose \method, a family of optimizers that employs multiple slow- and fast-moving first momenta or the gradient to track update dynamics across different time scales, for which we provide the first convergence guarantees. 
Empirically, for language-model pre-training, this eliminates the performance gap with \ddp, outperforming infrequent-communication baselines in perplexity and reducing iso-token wall-clock time relative to \ddp by $6\text{--}27\%$ on Ethernet interconnects. At the $720$M scale, \method reaches a target perplexity in $24\%$ fewer steps and $35\%$ less time than the single-momentum \ddp baseline. \method enables effective cross-datacenter training and training over wide geographic areas.
\end{abstract}

\section{Introduction}
\label{sec:intro}

The scalability of training infrastructure is impeded by the communication required for Distributed Data Parallelism (\ddp). Infrequent parameter-averaging strategies like Local SGD~\citep{LocalSGD} reduce this overhead, yet extensions to adaptive optimizers~\citep{LocalAdam,DiLoCoScalingLaws} show a performance gap relative to \ddp~\citep{Photon,DiLoCoScalingLaws}. \citet{DiLoCoScalingLaws} finds that infrequent averaging, even with Nesterov momentum at round boundaries~\citep{FedOPT}, underperforms \ddp for models up to $2.4$B parameters and worker counts exceeding $2$.

We hypothesize this gap stems from a \emph{timescale mismatch}. Optimizers use fast-moving momenta (low $\beta_1\approx 0.9$) that smooth high-frequency noise under \ddp but \emph{decay too rapidly} between infrequent synchronizations. This decay prevents a stable shared trajectory, leading to our central question:

\begin{quote}
\centering
\emph{
Can a distributed adaptive optimizer with $\beta$'s suited for infrequent communication close the performance gap with \ddp while providing convergence guarantees?
}
\end{quote}

We propose \method, which brings multi-momentum optimizers~\citep{AggMo,AdEMAMix} to the distributed, infrequent-communication regime. \method resolves the mismatch by using slow-moving momenta (e.g., $\beta \approx 0.999$) to preserve trajectory information across synchronizations while remaining responsive via a fast momentum. In its simplest quasi-hyperbolic form~\citep{QHM}, \method uses the current gradient as the fast momentum, adds no memory or communication overhead, and requires only one additional hyperparameter. Crucially, unlike methods that use a momentum-based outer optimizer~\citep{FedOPT,DiLoCo} at synchronization boundaries, \method needs \textbf{no extra memory buffers} or multiple outer hyperparameters.

Empirically, slow momentum acts as a regularizer, improving update alignment by increasing cosine similarity between worker pseudo-gradients~\citep{FedOPT}. This stability lets \method improve perplexity over low-communication baselines. Furthermore, \method matches or exceeds its \ddp analogue at larger scales, \textbf{closing the perplexity gap} for models up to $720$M parameters.

\introtakeawaybox{
\begin{enumerate}[leftmargin=*]
\item \textbf{A Provably Convergent Multi-Timescale Framework.} We introduce \method, the first framework to integrate multi-momentum strategies into distributed settings, with convergence guarantees for heterogeneous momentum timescales and synchronization frequencies.

\item \textbf{Closing the Performance Gap Efficiently.} \method matches synchronous \ddp, outperforming baselines in perplexity, reducing wall-clock time by $6\text{--}27\%$; at $720$M it reaches a target perplexity in $24\%$ fewer steps and $35\%$ less time than a DDP baseline.

\item \textbf{Noise Suppression and Information Retention.} \method's slow momentum preserves mutual information across rounds and reduces inter-worker momentum variance.

\item \textbf{Resilience to Infrequent Communication.} \method lowers the rate of change of parameters and momenta, improving tolerance to low communication frequencies.

\item \textbf{Alignment of Worker Trajectories.} \method increases cosine similarity of local worker update trajectories, which reduces worker drift and aligns the overall model update.

\end{enumerate}
}

\section{\fullmethod~(\method)}

\begin{figure}[h] \centering \begin{subfigure}[b]{0.48\textwidth} \centering \includegraphics[width=\linewidth]{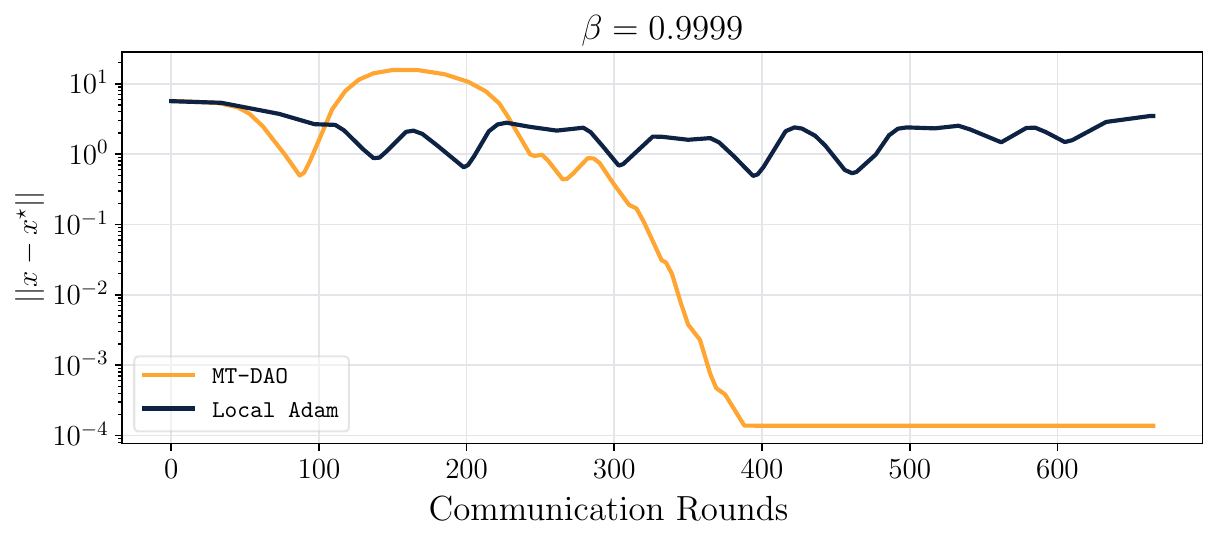} \caption{Distance to optimum vs. steps.} \label{fig:toy_dist_vs_steps_slow} \end{subfigure} \hfill %
\begin{subfigure}[b]{0.48\textwidth} \centering \includegraphics[width=\linewidth]{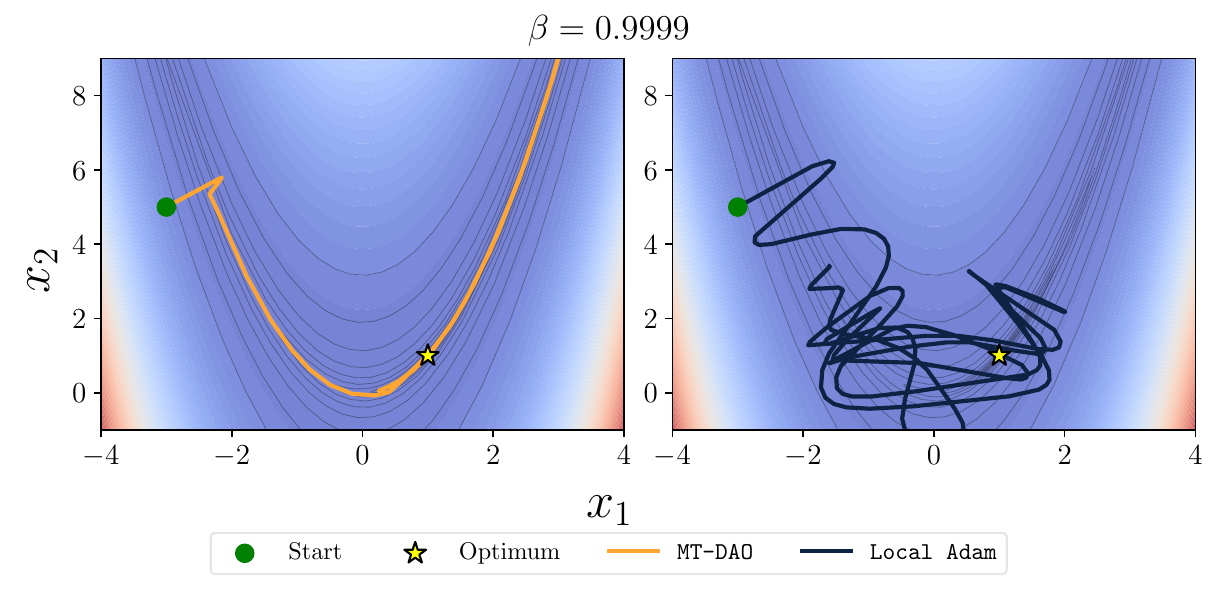} \caption{Contour plot of trajectories.} \label{fig:toy_contour_slow} \end{subfigure} \caption{To highlight the stability benefit of \method, we illustrate its performance on a toy non-convex problem. Crucially, under a high momentum decay of $\beta=0.9999$, prior stateful methods like \texttt{Local Adam}~\citep{LocalAdam} become unstable and fail to converge, whereas \texttt{MT-DAO} maintains its rapid and stable convergence. We optimize the non-convex Rosenbrock function $f(x_1, x_2) = (1 - x_1)^2 + 100(x_2 - x_1^2)^2$ with $M=256$ workers and IID Gaussian noise ($\sigma=2$).} \label{fig:toy_example_method_slow} \end{figure}

Our analysis begins by characterizing the conflict between optimizer momentum timescales and the long communication intervals of infrequent-communication training. We consider the standard setting. Let $M$ be the total number of workers and $K$ be the number of local updates performed by each worker per round. The goal is to minimize a global objective $f(x) := \frac{1}{M}\sum_{m=1}^M f_m(x)$ over the model parameters $x$, where each $f_m(x)$ is the local objective $\mathbb{E}_{\xi\sim\mathcal{D}_m} [F_m(x;\xi)]$ for a data sample $\xi$ drawn from data distribution $\mathcal{D}_m$. We posit that the performance degradation in this regime stems from a fundamental mismatch between the optimizer's memory and the communication period.

\subsection{Timescale Mismatch}
The first momentum in adaptive optimizers is an Exponential Moving Average (\texttt{EMA}). Let $u_t$ be the momentum, $\beta$ be the momentum decay factor, and $g_t$ be the gradient at step $t$. The momentum is then given by $u_t = \beta u_{t-1} + (1-\beta)g_t$. Its effective memory is quantified by its half-life, $\tau_{0.5}(\beta)=\frac{\ln 0.5}{\ln\beta}$, the number of steps until the momentum decays by $50\%$~\citep{AdEMAMix}. A typical $\beta_1=0.9$ yields $\tau_{0.5} \approx 6.6$ steps, suitable for frequent communication. A conflict arises when $K \gg \tau_{0.5}$, common in communication-efficient training ($K \in [32, 512]$). Unrolling the momentum update over $K$ local steps from a synchronized state $u_t$ gives: $ u_{t+K} = \beta^K u_t + (1-\beta)\sum_{k=0}^{K-1} \beta^k g_{t+K-k}$. The influence of the global state $u_t$ on the final local state $u_{t+K}$ decays with $\beta^K$. For $\beta_1=0.9$ and $K=32$, this factor is negligible ($\approx 0.03$). In such a setting, the worker becomes reliant on high-variance, potentially biased local gradients. For example, if noise is independent across workers, the variance of the final local momentum is $\text{Var}(u_{t+K}) = \frac{1-\beta}{1+\beta}(1-\beta^{2K})\sigma_m^2$~(see \cref{app:sec:small_derivations}), with $\sigma_m^2$ being the gradient variance of worker $m$. As $\beta \to 1$ the factor $\frac{1-\beta}{1+\beta}$ suppresses variance. For $\beta \to 0$, variance approaches $\sigma_m^2$, exposing local updates to noise-induced instability.

An alternative interpretation of this memory decay is offered by information theory, which quantifies the preserved signal between the initial global momentum $U_t$ and the final local momentum $U_{t+K}$ via their mutual information, $I(U_{t+K}; U_t)$. By modeling the local updates as a linear process $U_{t+K} = \beta^K U_t + L$, where $L$ is the accumulated local gradient noise, a closed-form expression can be derived when assuming Gaussian distributions for the states and noise with covariances $\Sigma_{U_t}$ and $\Sigma_L$ respectively. The mutual information is $I(U_{t+K}; U_t) = \frac{1}{2}\log \det(I + \beta^{2K} \Sigma_{U_t} \Sigma_{L}^{-1})$~(see \cref{app:sec:small_derivations}). As $\beta^K \to 0$, mutual information vanishes, implying statistical independence with respect to the momentum at the start of the interval. As $\beta^K\to1.0$, the initial signal is preserved.

\subsection{The Challenge of High-$\beta$ Optimizers}

Although both arguments above encourage the use of large $\beta$\ values as a solution to this timescale mismatch problem, previous work has shown that high-momentum optimizers are often unfeasible in practice~\citep{AggMo}. Without modification, they are insufficiently responsive to changes in the loss landscape and are prone to oscillations~(see the top of \cref{fig:quadratic_comparison_qhm_vs_mom}), yielding subpar performance. 

\begin{figure}[H]
    \centering
    \includegraphics[width=1.0\linewidth]{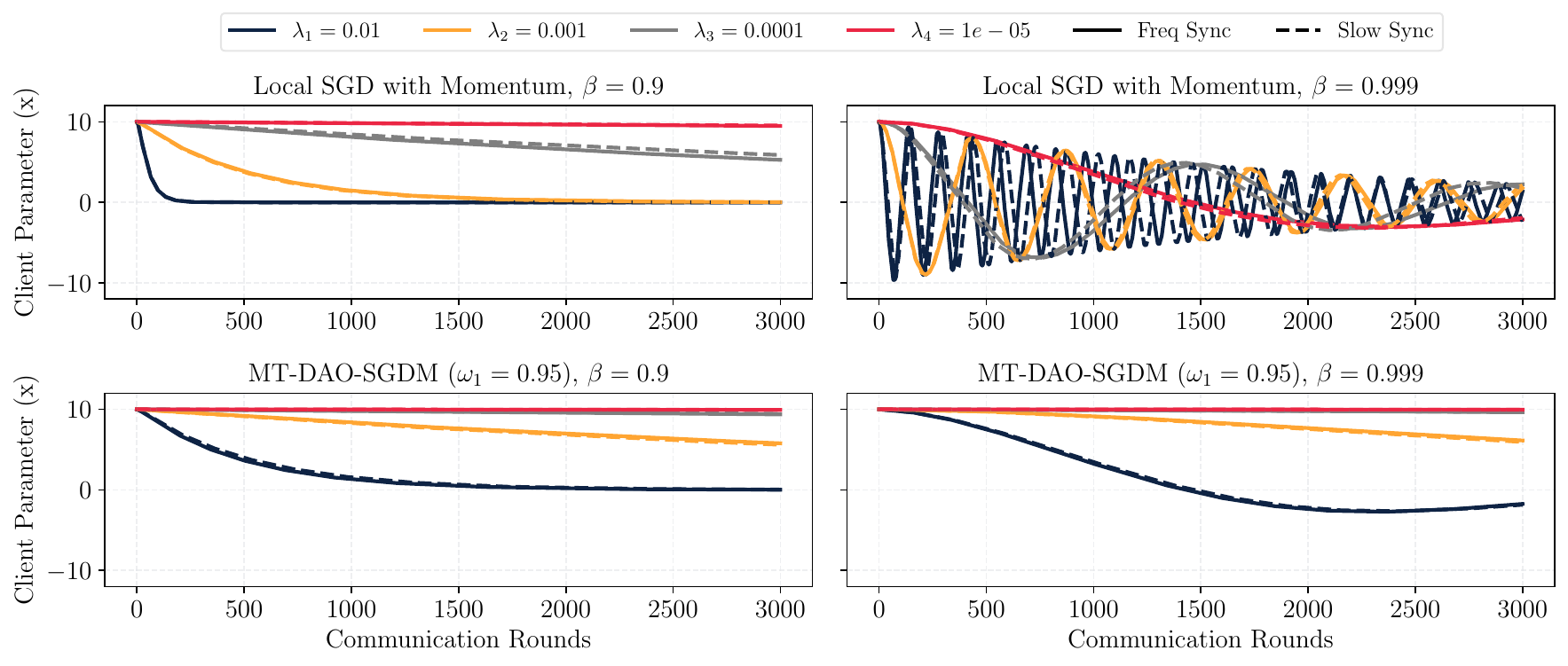}
    \caption{Comparison of Local SGDM with standard momentum \textbf{(top)} and \methodsgdm ($N=1$ momentum, $\omega_1=0.95$) \textbf{(bottom)} for the function $f(x;\lambda)=\frac{1}{2}{\lambda x^2}$ with $x \in \mathbb{R}$ for various parameters controlling the rate of change $\lambda$ and and sync frequencies (frequent: solid, infrequent/slow: dashed). While both optimizers are stable at low momentum ($\beta=0.9$), at high momentum ($\beta=0.999$) Local SGD with standard momentum becomes unstable for high $\lambda$ while \methodsgdm remains stable.}
    \label{fig:quadratic_comparison_qhm_vs_mom}
\end{figure}

This instability motivates using multi-momentum methods~\citep{AggMo,QHM,AdEMAMix}. Such methods compose the optimizer update as a linear combination of slow and fast-moving first momenta, or the gradient in the case of Quasi-hyperbolic methods~\citep{QHM}. This avoids common pitfalls of high-momentum methods by responding to changes in the loss landscape via the fast momentum/gradient. Recent works~\citep{AdEMAMix,BenchmarkingOptimizersLLM} have shown that such optimizers can provide \texttt{SOTA} results, outperforming popular optimizers such as \adam~\citep{AdamW}, \muon~\citep{Muon}, and \dion~\citep{Dion}.

\subsection{The \method Method and Algorithm}\label{subsec:method_algorithm}

\begin{algorithm}[h]
\caption{\methodadam, local bias correction omitted to save space.}
\label{alg:method_adam}
\footnotesize
\begin{algorithmic}[1]
 \Require \textbf{Model tensors, hyper-parameters} \\
        \quad $x_0 \in \mathbb{R}^{d}$, $\{\bar{u}^{j}_{-1}\}_{j=1}^{N} \in (\mathbb{R}^d)^N$, $\bar{v}_{-1} \in \mathbb{R}^d$ — initial params, $N$ first momenta, second momentum\\
        \quad $\{\beta_{1,j}\}_{j=1}^N, \beta_2 \in [0,1)$ — decay rates for each momentum state \\
        \quad \textcolor{purple}{$\{\omega_j\}_{j=1}^N \in [0,1]$} — convex combination coefficients for first momenta, $\sum_{j=1}^N w_j \leq 1.0$ \\
        \quad $\rho \in \mathbb{R}_{+}$, $\{\eta_t\}_{t=0}^{T-1}$ — clipping radius, learning-rate schedule\\
        \quad $T,M \in \mathbb{N}_{+}$ — total optimization steps and number of workers\\
        \quad $K_x, \{K_j\}_{j=1}^{N}, K_v \in (\mathbb{N}_{+})^{N+2}$ — communication periods for parameters and states \\
        \quad $\texttt{OuterOpt}:\mathbb{R}^d\to\mathbb{R}^d$ — update params using an outer optimizer, averaging by default 
    
 \Ensure $x_T,\;\{u_{T-1}^{j}\}_{j=1}^{N},\;v_{T-1}$

 \State \textbf{for each worker} $m$: initialize $x_0^m, \{u_{-1}^{j,m}\}, v_{-1}^m$
 \For{$t = 0,\dots,T-1$}
    \ForAll{workers $m=0,\dots,M-1$ \textbf{in parallel}}
        \State $\hat{g}_t^m \gets \clip(\nabla F(x_t^m;\xi_t^m),\rho)$ \hfill\textcolor{gray}{\scriptsize clipped stochastic gradient}

        \For{$j = 1$ \textbf{to} $N$} \hfill\textcolor{gray}{\scriptsize update N first momenta}
            \State \textcolor{purple}{$u_t^{j,m}  \gets \beta_{1,j} \bar{u}_{t-1}^{j} + (1-\beta_{1,j}) \hat{g}_t^m$ }
            \State \textcolor{purple}{$\bar{u}_{t-1}^{j}  \gets$ \textbf{if} $(t \bmod K_j = 0)$ \textbf{then} $\mathbb{E}_m[u_{t-1}^{j,m}]$ \textbf{else} $u_{t-1}^{j,m}$ \hfill\textcolor{purple}{\scriptsize sync $u^j$ every $K_j$ steps} }
          
        \EndFor
        \State $v^m_t \gets \beta_2\bar{v}_{t-1} + (1-\beta_2)(g^m_t)^2$
        \State $\bar{v}_{t-1} \gets$  \textbf{if} $(t \bmod K_v = 0)$ \textbf{then} $\mathbb{E}_m[v_{t-1}^{m}]$ \textbf{else} $v_{t-1}^{m}$ \hfill\textcolor{gray}{\scriptsize sync $v$ every $K_v$ steps}

        \State \textcolor{purple}{$\Delta_t^m \gets \frac{1}{\sqrt{v_t^m}+\epsilon} \left[ (1 - \sum_{j=1}^N \omega_j)  \hat{g}_t^m + \sum_{j=1}^N \omega_j u_t^{j,m} \right]$} \hfill\textcolor{purple}{\scriptsize form combined update direction}
        \State $x_{t+1}^m \gets \bar{x}_t - \eta_t \Delta_t^m$
        \State $\bar{x}_t \gets$  \textbf{if} $(t \bmod K_x = 0) $ \textbf{then} \texttt{OuterOpt}($\mathbb{E}_m[x_t^{m}]$ ) \textbf{else} $x_t^{m}$ \hfill\textcolor{gray}{\scriptsize sync $x$ every $K_x$ steps}
  
    \EndFor
 \EndFor
\end{algorithmic}
\end{algorithm}
Based on this analysis, we formalize the \method framework in Algorithm \ref{alg:method_adam} for \adam with a variant for \adopt in \cref{alg:method_adopt} and one for SGD with Momentum~(\texttt{SGDM}) presented in \cref{alg:method_sgdm_non_prob}. It accommodates adaptive optimizers with $N$ first-order momenta $\{u^{j}\}$ and a single second-order momentum $v$. The parameter update is driven by a convex combination with hyper-parameters $\{\omega_j\}$ of these $N$ preconditioned momenta and the preconditioned current gradient, which receives the remaining weight $1 - \sum_{j=1}^N \omega_j$. We highlight these additions in \textcolor{purple}{purple}. This inclusion of the current gradient term effectively implements the Quasi-hyperbolic Momentum~(\qhm) structure within this generalized multi-momentum framework. The $\texttt{OuterOpt}$ procedure represents arbitrary parameter optimizers such as Federated Averaging~\citep{fedavg}, \nesterov Momentum~\citep{FedMOM}, or $\texttt{FedOPT}$~\citep{FedOPT}. Unless stated otherwise, our analysis and arguments refer to using averaging to align with previous converge analyses~\citep{LocalAdam,DES-LOC}.  \methodadam reduces communication costs by $(\frac{1}{K_x}+\sum_{j=1}^N \frac{1}{K_j} + \frac{1}{K_v})^{-1}$ over \ddp.

This generalized framework recovers previous distributed adaptive optimizers~\citep{LocalSGD,DiLoCo,LocalAdam,DES-LOC}. It also introduces \textbf{the first-ever formulations for provably convergent distributed variants of multi-momentum optimizers}~\citep{AggMo,QHM,AdEMAMix}. \Cref{fig:quadratic_comparison_qhm_vs_mom}~(bottom) shows an example of \methodsgdm converging for both high and low $\beta_1$ with a quasi-hyperbolic formulation while the \localsgd with momentum averaging method fails for high $\beta_1$. To highlight the stability of \methodadam, \Cref{fig:toy_example_method_slow} illustrates its convergence on a common toy non-convex problem~\citep{AdEMAMix} under high momentum ($\beta=0.9999$), a setting where prior provably convergent methods like \localadam~\citep{LocalAdam} become unstable and do not reach the optimum.

\section{Convergence Guarantees for \method}\label{sec:theory}

This section provides a theoretical convergence analysis for the proposed \method approach using the SGDM optimizer. The analysis, detailed in \cref{app:proof-sgdm}, relies on the following standard assumptions.

\begin{asp}[Lower bound  and smoothness]\label{ass:smooth}
    The overall loss function $f\colon\R^d\to\R$ is lower bounded by some $f^{*} \in \mathbb{R}$ and all local loss functions $f_m$ are $L$-smooth:
    $$\|\nabla f_m(x) - \nabla f_m(y)\| \leq L \|x-y\|, \quad \text{for any } x,y\in\R^d.$$ 
\end{asp}

\begin{asp}[Unbiased noise with bounded stochastic variance]\label{ass:boundgrad}
    The stochastic gradient $g^m$ of local loss function $f_m$ computed by machine $m$ is unbiased and the noise has bounded variance:
    $$\E[g^m] = \nabla f_m(x), \quad \E[\|g^m - \nabla f_m(x)\|^2] \le \sigma^2, \quad \text{for any } x\in\R^d.$$
\end{asp}

\begin{asp}[Bounded heterogeneity]\label{ass:het}
    For any $x\in\R^d$, the heterogeneity is bounded by
    $$\textstyle\frac{1}{M}\sum_{m=1}^M\|\nabla f_m(x) \|^2 \le G^2 + B^2\|\nabla f(x)\|^2.$$
\end{asp}

These are standard assumptions in smooth non-convex optimization \citep{Yu2019,pmlr-v119-karimireddy20a,wang2021fieldguidefederatedoptimization,Yuan2022}, covering homogeneous data as a special case ($G^2=0, B^2=1$). For analytical tractability, we model periodic synchronization every $K$ steps as a probabilistic event. Model parameters are averaged with probability $p_x = 1/K_x$, the $j$-th momentum is averaged with probability $p_j = 1/K_j$. The gradient is treated as a momentum with $\beta=0$.

\begin{thm}
Let Assumptions \ref{ass:smooth}, \ref{ass:boundgrad} and \ref{ass:het} hold. Then, choosing the step size $\eta = \min(\eta_0, \frac{1}{\sqrt{T}})$ where $\eta_0 \eqdef 1/( 4L \max(\beta_{\omega}, 6\sqrt{\psi\max(1,B^2-1)}) )$ with constants
\begin{equation}\label{eta-psi}
\textstyle
\beta_{\omega} \eqdef \sum_{j=1}^N \frac{\omega_j \beta_j}{1-\beta_j},
\quad\text{and}\quad
\psi \eqdef
\frac{4(1-p_x)}{p_x^2} \sum_{j=1}^N \omega_j\frac{(1-\beta_j)(1-p_j)}{1-(1-p_j)\beta_j}
\end{equation}
the average iterates $x_t = \E_m[x_t^m]$ of \methodsgdm converge with the following rate:
\begin{equation}\label{rate-sgdm}
\textstyle
\frac{1}{T}\sum_{t=0}^{T-1}\E{\|\nabla f(x_t)\|^2}
\le \frac{4}{\sqrt{T}}\left(f(x_0) - f^*
+ \frac{L\sigma^2 }{2M} \right)
+ \mathcal{O}\left(\frac{1+\beta_{\omega}^2 + \psi}{T}\right).
\end{equation}
\end{thm}

The derived bound in \eqref{rate-sgdm} achieves the optimal $\mathcal{O}(1/\sqrt{T})$ asymptotic rate for smooth non-convex stochastic optimization \citep{arjevani2023lowerbound}. Distributed factors, such as client drift and data heterogeneity, are contained within the step-size constraint and the higher-order $\mathcal{O}(1/T)$ term, thus not affecting the asymptotic rate. The step size $\eta$ is constrained by $\beta_{\omega}$ and $\psi$. The dependence $\psi = \mathcal{O}(1/p_x^2)$ shows that model synchronization frequency $p_x$ is critical. The impact of momentum synchronization is nuanced: reducing a momentum's sync frequency $p_j$ increases its contribution to $\psi$, but this is modulated by its decay rate $\beta_j$. This implies "slower" momenta (larger $\beta_j$) are more robust to infrequent synchronization. \textbf{This analysis reveals a trade-off: large $\beta_j$ values constrain the step size via $\beta_{\omega}$ but reduce the communication penalty in $\psi$}. Furthermore, synchronizing only the model always (i.e., $p_x=1, p_j=0$) is algorithmically equivalent to synchronizing only the momenta always (i.e., $p_x=0, p_j=1$). In the boundary case where only model parameters are synced ($p_x=1, p_j=0$), $\psi=0$ and the rate recovers that of mini-batch SGD \citep{Liu2020}.

\section{Experimental Design}
\label{sec:exp_design}

Building on our analysis, our experimental design answers the following research questions:
\begin{itemize}[noitemsep,topsep=0pt,parsep=2pt,partopsep=0pt]
    \item[\textbf{RQ1}] Does \method reduce momentum noise and preserve mutual information, as predicted?
    \item[\textbf{RQ2}] Does \method better preserve task performance when decreasing communication frequency?
    \item[\textbf{RQ3}] How does \method perform against \ddp and prior communication-efficient optimizers?
    \item[\textbf{RQ4}] How does slow momentum affect local optimization trajectories between synchronizations?
\end{itemize}

\subsection{Setup}\label{sec:exp_setup}

\textbf{Models and Data.}
We use peri-norm~\citep{PeriLayerNorm} \texttt{GPT}-style transformer models of $16$M, $125$M, and $720$M parameters (\cref{tab:model_architectures}). The $16$M model is used for hyperparameter sweeps and qualitative investigations, while the $125$M and $720$M models are used for scaling experiments and baseline comparisons. All models are trained with a sequence length of $2048$ on the \texttt{SmolLM2} mixture~\citep{SmolLM2}. We evaluate all models ($16$M, $125$M, $720$M) using validation perplexity on a held-out $10\%$ portion of the training mixture. For further details, please see \cref{app:exp_details}.

\textbf{Optimizers and Tuning Methodology.}
We use the \adopt optimizer, a variant of \adam whose convergence rate is independent of the second-momentum decay rate $\beta_2$ and preserves performance~\citep{ADOPT}; we fix $\beta_2=0.9999$ to isolate the first momentum dynamics (governed by $\beta_1$ and $\omega$). We use the \texttt{CompleteP} parameterization for one-shot transfer of the learning rate (LR) from small to large models~\citep{CompleteP}. For each combination of convex coefficients ($\omega$'s) and momentum decays ($\beta$'s), we tune the learning rate on the $16$M model and transfer the optimal hyperparameters to larger models directly without ever re-tuning ($\omega$'s) and ($\beta$'s). To establish strong \ddp baselines, we tune $\omega,\beta_1$ parameters in the \ddp setting and reuse them  for \method. For complete details see \cref{app:subsec:hyperparameter_sweeping_procedure}. We always use quasi-hyperbolic \method~($N=1$) which does not require additional memory and reduces comms costs by $(\frac{1}{K_x}+\frac{1}{K_1}+\frac{1}{K_v})^{-1}$ over \ddp.

\textbf{Baselines.}
We compare \method against: the base optimizer (\adopt) with \ddp and \ddp analogues to \method such as Quasi-hyperbolic Momentum~(\qhm)~\citep{QHM}. For communication-efficient baselines we use the provably convergent and stateful \localadam~\citep{LocalAdam} approach. We also compare against using Nesterov momentum as the outer optimizer~\citep{DiLoCoScalingLaws}. We evaluate ML performance for communication-efficient methods under the same, fixed synchronization frequency. Unless otherwise stated we use $K=K_x=K_1=K_v=32$ steps, based on prior work finding a practical balance of performance efficiency~\citep{DiLoCoScalingLaws}. We split the dataset in an $\texttt{IID}$ fashion across $4$ workers using $1$ H100 per worker.

\textbf{Other Metrics} We analyze flattened models/momenta $s_t \in \mathbb{R}^d$ using several metrics. The \textbf{relative change} over $K$ steps is measured as $\|s_{t+K}-s_t\|_2/\|s_t\|_2$. To quantify the dispersion among $M$ worker vectors, we compute the \textbf{cross-worker variance}, defined as $\frac{1}{M} \sum_{m=1}^M \|s_m - \bar{s}\|_2^2$. The statistical dependency between two random vectors at different timesteps is captured by their \textbf{mutual information}, $I(U_{t+K}; U_t)$. Finally, we measure alignment between vectors using \textbf{cosine similarity}.

\section{Evaluation}\label{sec:evaluation}

This section empirically validates \method, showing its slow momentum preserves information and aligns workers (\cref{sec:eval:momentum_noise_mutual_information,sec:eval:alignment}), which improves stability under infrequent communication (\cref{sec:eval:resilience}) and allows it to close the performance gap with \ddp at scale (\cref{sec:eval:baseline_comparison}).
\subsection{\method Reduces Momentum Noise and Preserves Mutual Information~(\textbf{RQ1})}\label{sec:eval:momentum_noise_mutual_information}

\begin{figure}[]
    \centering
    {\includegraphics[width=0.49\columnwidth]{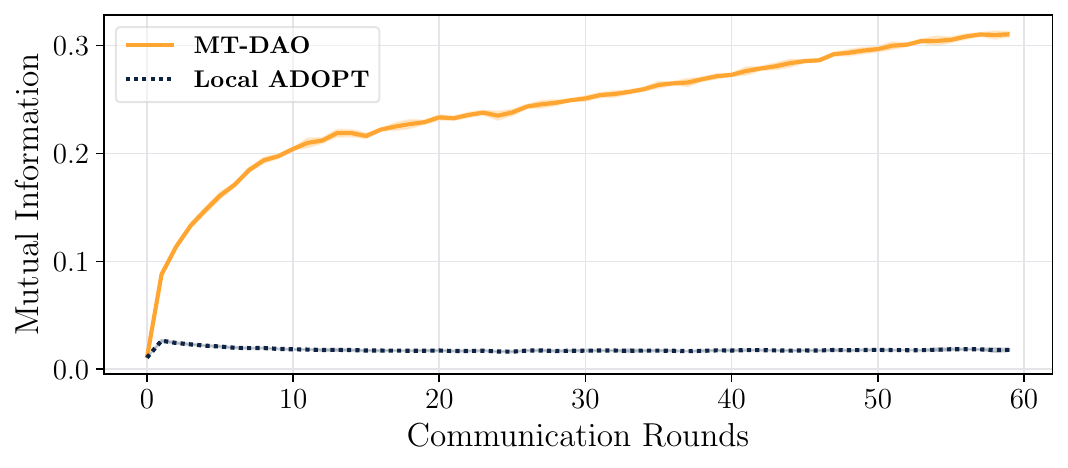}} \hfill
    {\includegraphics[width=0.49\columnwidth]{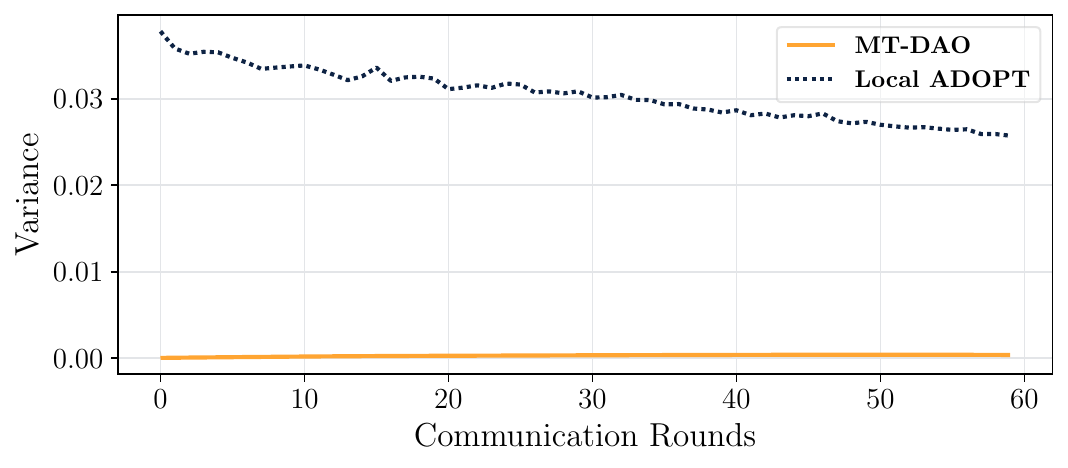}}
    \caption{
A comparison of \method ($\beta_1=0.999$) versus a \localadopt baseline ($\beta_1=0.95$) with a communication frequency of $K=32$. For each communication round, we plot metrics computed between the momentum at the start ($t$) and end ($t+K$) of the round. \method's slow momentum \textbf{preserves mutual information}, $I(U_{t}; U_{t+K})$, across rounds while the baseline's momentum decays losing the global optimization direction (left). Furthermore, \method \textbf{reduces inter-worker momentum variance}, $\text{Var}(u_{t+K})$, indicating greater stability against local noise (right).
}
    \label{fig:mi_and_variance}
\end{figure}

We now empirically validate the motivation of \method. Our results in \cref{fig:mi_and_variance} demonstrate that the slow-momentum of \method both preserves information about the global optimization direction across communication rounds and suppresses the variance induced by local updates.

\takeawaybox[Slow Momentum Is Preserved:]{%
The slow momentum in \method preserves its direction across communication rounds. This directional memory also reduces the influence of local gradient noise, leading to lower momentum variance across workers and \textbf{a more stable optimization path}.
}

\subsection{\method Is Resilient to Infrequent Communication~(\textbf{RQ2})}\label{sec:eval:resilience}

We now investigate if \method provides greater resilience against infrequent synchronization, as predicted by our analysis in \cref{sec:theory} showing that reducing the communication frequency of momenta with higher $\beta$ has a diminshed impact on the step size.

\begin{table}[H]
\caption{Demonstration of how parameter synchronization period ($K_x$) affects final perplexity for two \method configurations
with momentum periods $K_1 = K_v=16$ for our $16$M models. Values show the percentage increase in validation perplexity over the $K_x=16$ baseline. Higher $\beta$ leads to \textbf{less performance degradation as $K_x$ increases}.
}\label{tab:sync_freq_perf} 
\centering
\resizebox{0.6\textwidth}{!}{%
\begin{tabular}{@{}lccccccc@{}}\toprule
$\mathbf{K_x}$ & $\mathbf{16}$ & $\mathbf{32}$ & $\mathbf{64}$ & $\mathbf{128}$ & $\mathbf{256}$ & $\mathbf{512}$ & $\mathbf{1024}$ \\
\midrule
$\mathbf{\beta_1 = 0.99}$ & $ (37.72) $ & $ +1.7\% $ & $ +3.0\% $ & $ +3.9\% $ & $ +5.1\% $ & $ +5.6\% $ & $ +6.2\% $ \\
$\mathbf{\beta_1 = 0.995}$ & $ (37.65) $ & $ +1.0\% $ & $ +1.6\% $ & $ +3.2\% $ & $ +2.8\% $ & $ +3.4\% $ & $ +3.7\% $ \\
\bottomrule\end{tabular}
}
\end{table}

\Cref{tab:sync_freq_perf} shows that a \method configuration with a higher momentum decay $\beta_1$ suffers less performance degradation as the parameter synchronization period $K_x$ increases. This improved resilience can be explained by the reduced rate of change in the model parameters, which is quantified in \Cref{fig:heatmap_qh}. The underlying principle is that the global averaging step is most effective when the worker models have diverged minimally. Let $x_t$ be the synchronized model at the start of a round. After $K_x$ local steps, worker $m$ arrives at state $x_{t+K_x}^m$. \Cref{fig:heatmap_qh} shows that high $(\beta_1, \omega_1)$ values reduces the local model change, $\mathbb{E}_m[\|x_{t+K_x}^m - x_t\|_2]$. A smaller per-worker model change bounds the variance across the set of workers ($\{x_{t+K_x}^m\}_{m=1}^M$), mitigating local drift. This enhances convergence robustness because workers compute gradients on models that are closer to the global mean, making their local updates more relevant to the central objective~\citep{FedProx}.

\begin{figure}[]
    \centering
    {\includegraphics[width=0.49\columnwidth]{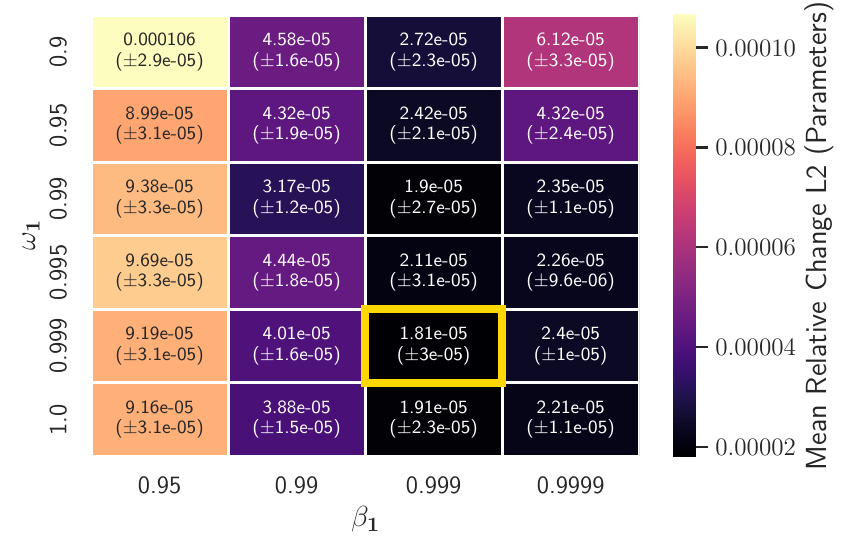}} \hfill
    {\includegraphics[width=0.49\columnwidth]{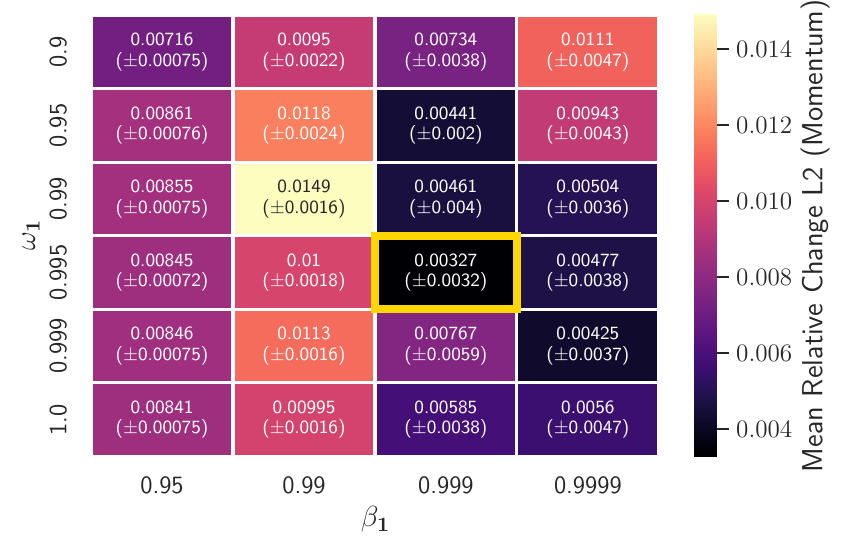}}
    \caption{Mean relative L2 change and standard deviation across communication rounds of (left) model parameters and (right) the first momentum state, as a function of momentum decay ($\beta_1$) and weight ($\omega_1$). In both cases, \method shows a significantly \textbf{reduced} relative rate of change with high $(\beta_1, \omega_1)$ (minimum in gold), which \textbf{reduces worker drift and thus makes parameter averaging more effective}. Each point on the grid corresponds to a configuration evaluated with its own independently tuned learning rate. \localadopt corresponds to ($\beta_1=0.95,\omega_1=1.0)$. 
    }
    \label{fig:heatmap_qh}
\end{figure}

\takeawaybox[Slow Momentum as Anchor:]{%
Long-term momentum (high $\beta_1$ and $\omega$) reduces the rate of change of parameters and optimizer states. This stability ensures worker models diverge less prior to synchronization, \textbf{which reduces the performance impact of infrequent synchronization}.
}

\subsection{\method Outperforms Prior Low-comms Optimizers and Matches \ddp~(\textbf{RQ3})}\label{sec:eval:baseline_comparison}

We evaluate \method on $16$M, $125$M, and $720$M parameter language models against other baselines.
We report validation perplexity as a function of both training tokens and wall-clock time. 
Timings are measured on $4$ cloud H100s connected via $50$--$100$~Gbit/s Ethernet, including constant implementation overheads, and accounting for communication–computation overlap in \ddp. These measurements are specific to this hardware; \cref{app:sec:wall_time_model} provides a bandwidth model that compares communication-efficient methods to \ddp across a wider range of interconnects. When reporting time-to-target perplexity, we give improvements in both wall-clock time and training tokens. 

\begin{figure}[]
    \centering
    {\includegraphics[width=0.32\columnwidth]{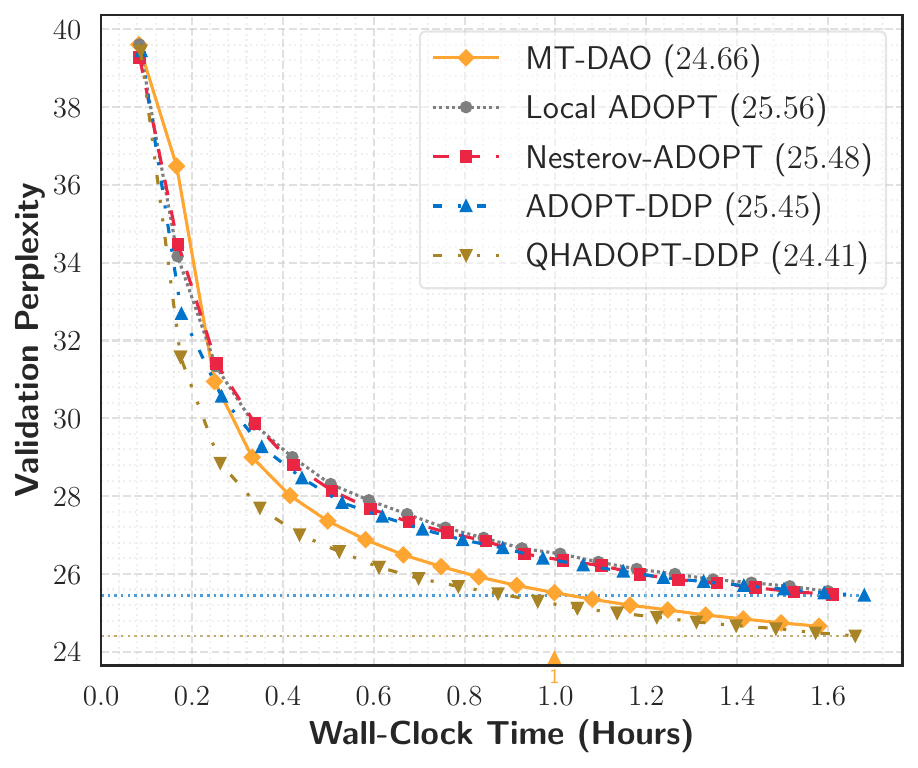}} \hfill
    {\includegraphics[width=0.32\columnwidth]{{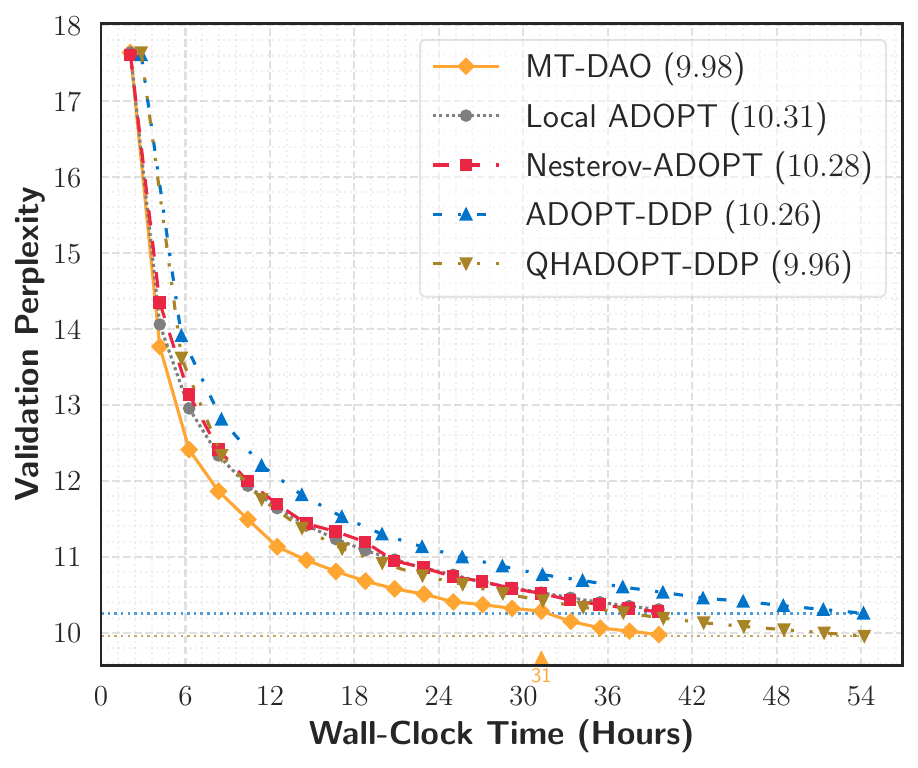}}} \hfill
   {\includegraphics[width=0.32\columnwidth]{{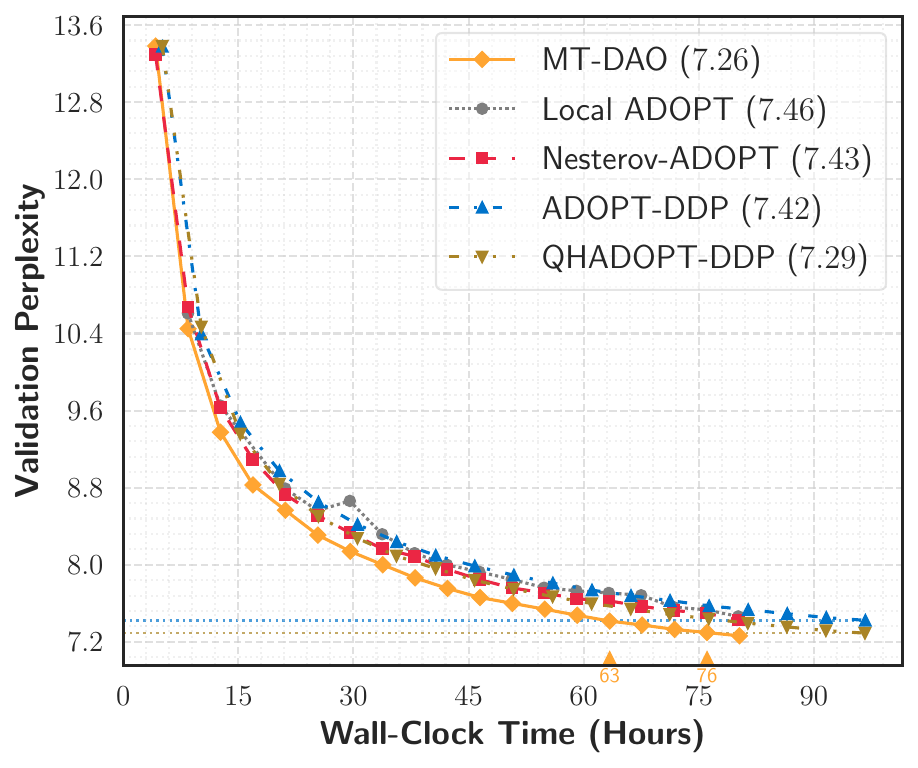}}} \hfill
   \subfloat[$16$M]{\includegraphics[width=0.32\columnwidth]{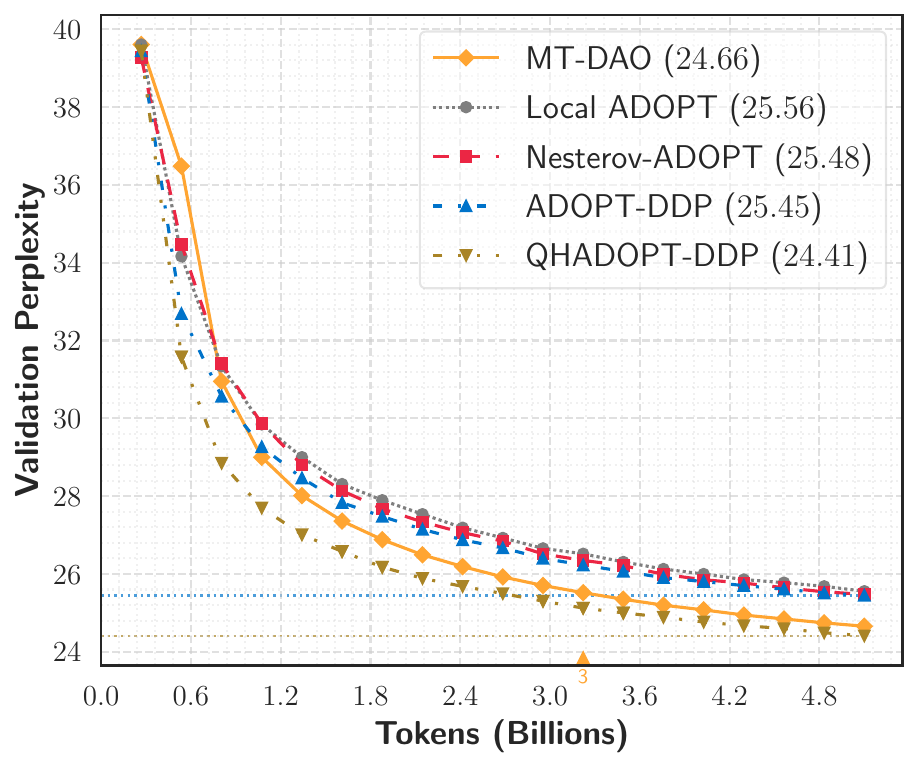}} \hfill
    \subfloat[$125$M]{\includegraphics[width=0.32\columnwidth]{{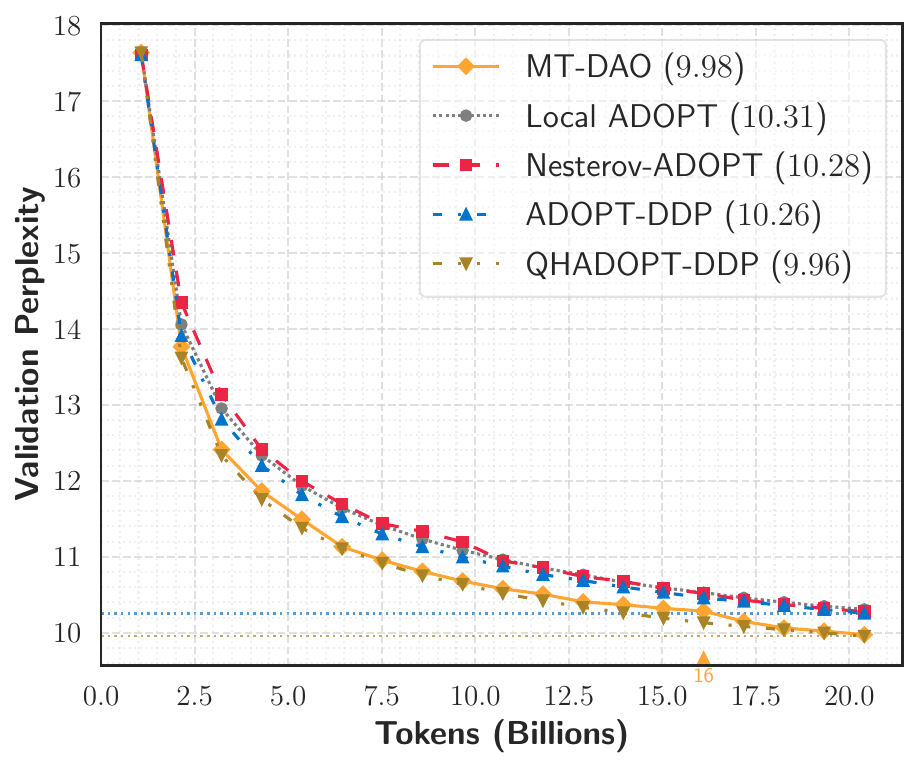}}} \hfill
   \subfloat[$720$M] {\includegraphics[width=0.32\columnwidth]{{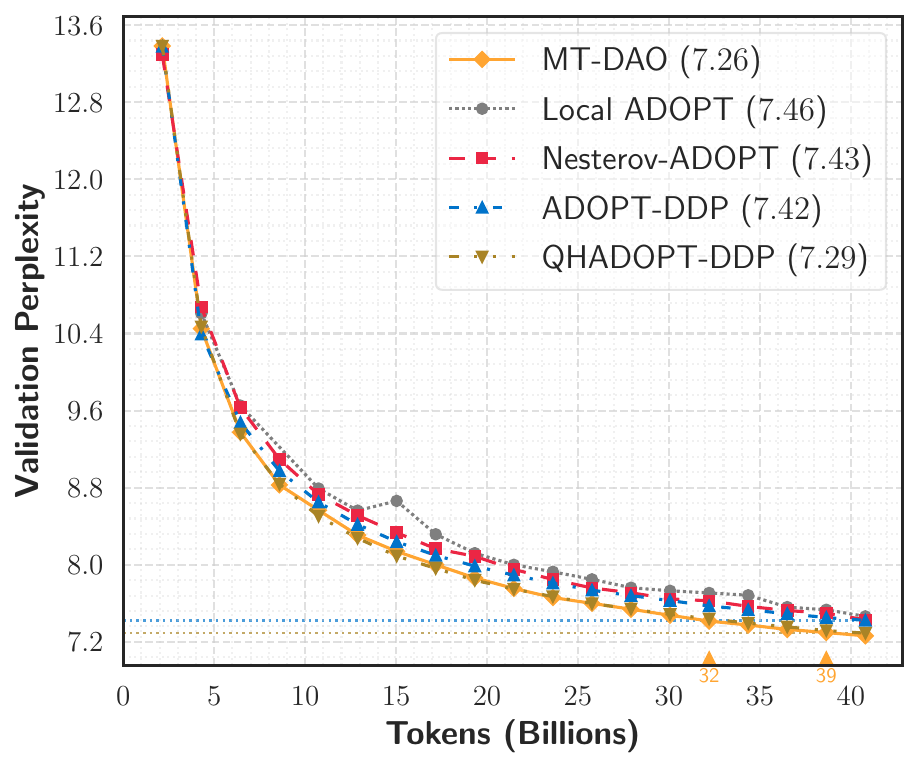}}} 
      \caption{Validation perplexity versus wall-clock time and training tokens for \method and baselines on models of size (a)~16M, (b)~125M, and (c)~720M models. Horizontal lines denote the two \ddp baselines (\texttt{ADOPT-DDP} and \texttt{QHADOPT-DDP}). For each non-\ddp method, a colored marker on the x-axis marks the earliest point at which its curve attains a lower/equal perplexity to a \ddp variant.}

    \label{fig:val_perplexity_time_16M_125M_720M}
\end{figure}

Across all scales, \method consistently improves over \texttt{ADOPT-DDP} and \texttt{Local ADOPT} in both tokens and time, closing the gap to synchronous training and reducing end-to-end wall clock by $\mathbf{6\text{--}27\%}$. At $720$M, relative to single-momentum \ddp, \method reaches the same perplexity in $\mathbf{24\%}$ fewer tokens and $\mathbf{35\%}$ less time. Relative to \texttt{QHADOPT-DDP}, \method trails at $16$M, matches at $125$M, and at $720$M reaches the \texttt{QHADOPT-DDP} target perplexity about $8\%$ faster in wall-clock and $\approx 5\%$ fewer tokens. The additional improvements in time are due to \method communicating $\mathbf{10 \times}$ less than \ddp. The outer \texttt{Nesterov} baseline performs better than \texttt{Local ADOPT} in our setting yet remains below \method and \ddp; matching the findings of \citet[][Table~4]{DiLoCoScalingLaws}, which reported underperformance relative to \ddp by $0.2\%$ to $1.7$\% at the $550$M to $1.3$B scale. Mechanistically, \nesterov coalesces per-round gradients and preserves them according to the half-life of the \emph{outer} momentum, whereas \method implements a finer-grained \emph{inner} multi-timescale modification. We examine how these choices shape worker update trajectories in the next section, however, we do note that \method~($N=1$) does not require the additional momentum buffer of Nesterov and has only one additional hyperparameter to tune instead of two.

\takeawaybox[Improved Performance and Efficiency at Scale:]{%
\method improves performance w.r.t all baselines across model scales, \textbf{closing the performance gap to \ddp.}
}


\subsection{\method Aligns Worker Update Trajectories(\textbf{RQ4})}\label{sec:eval:alignment}
Having established the performance benefits of \method, we now investigate the underlying mechanism. We hypothesize that the slow momentum reduces worker drift by keeping the optimization trajectories of individual workers aligned with the global optimization direction. To validate this, we measure the cosine similarity between key optimization vectors. We define the per-round local update as the "pseudo-gradient" ($\Delta^m = x^m_{t+K} - x^m_{t}$), and the global pseudo-gradient as the average of local ones~\citep{FedOPT}. To provide a comprehensive comparison, we define the "global momentum" for each method: for \method and \texttt{Local ADOPT}, it is the average of worker momenta at the end of a round, while for the \texttt{Nesterov} variant, it is the state of the outer \texttt{Nesterov} momentum.

\begin{figure}[]
    \centering
    \includegraphics[width=\linewidth]{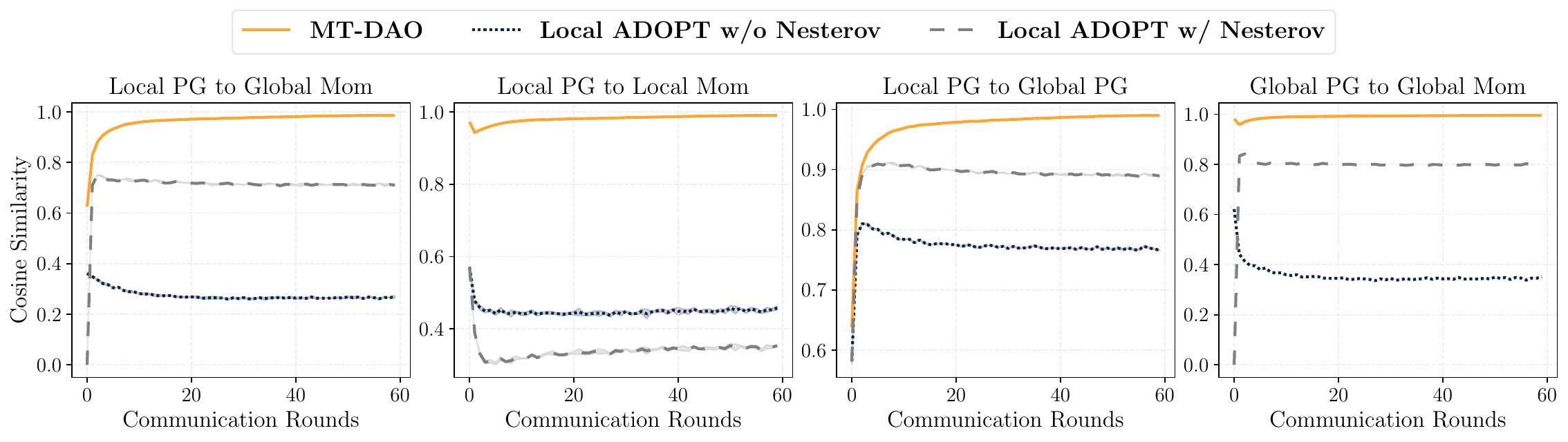}
    \caption{A comparison of update vector alignments for \method ($\beta_1=0.999, \omega_1=0.98$) versus \texttt{Local ADOPT} ($\beta_1=0.95$), and \texttt{Local ADOPT} ($\beta_1=0.95$) with \nesterov. Cosine similarity is measured between: (1) the local pseudo-gradient and global momentum, (2) the local pseudo-gradient and the local momentum, (3) the local and global pseudo-gradients,  (4) the local and global momentum. 
    Pseudo-gradient and momentum have been abbreviated as \textit{PG} and \textit{Mom}.}
    \label{fig:update_trajectories}
\end{figure}

The results in \cref{fig:update_trajectories} show that \method achieves near-perfect alignment (cosine similarity $>0.95$) across all four metrics. This indicates that: (1) each worker's update is consistent with its own momentum history (Local PG to Local Mom), (2) workers are in strong agreement with each other (Local PG to Global PG), and (3, 4) both local and global updates are aligned with the long-term global trajectory (Local/Global PG to Global Mom). This demonstrates that the slow momentum acts as a regularizer, ensuring all workers maintain a stable and shared optimization path.

In contrast, the \nesterov outer optimizer presents mixed results. As an \ema of global pseudo-gradients, it is better aligned to the global pseudo-gradient than the \localadopt momentum and it improves the alignment between the local and global pseudo-gradients compared to standard \texttt{Local ADOPT}. However, it never reaches the degree of alignment of \method in any metric. 


\takeawaybox[Slow Momentum as Regularizer:]{%
\method's slow momentum acts as a regularizer for each worker, ensuring that \textbf{local updates remain aligned with their history and the global trajectory.} 
}
\section{Related Work}
\label{sec:related_work}

Standard Distributed Data Parallelism's (\ddp) per-step synchronization creates a communication bottleneck~\citep{Horovod}. This is mitigated by two orthogonal strategies: payload compression and infrequent synchronization. Compression shrinks transmissions via quantization~\citep{QSGD_Alistarh_2017}, sparsification~\citep{DeepGradientCompression_Lin_2017}, mixes thereof~\citep{CocktailSGD}, low-rank updates~\citep{LDAdam_Robert_2023}, or communicating select momentum components~\citep{peng2024decoupled}. Our work advances infrequent synchronization~\citep{LocalSGD,fedavg} which allows local updates between communications and is complementary to compression.

Adapting stateful optimizers like \adam to infrequent synchronization is not straightforward. \localadam~\citep{LocalAdam} provided the first convergence proofs at the cost of synchronizing all optimizer states. \citet{DiLoCo,DiLoCoScalingLaws} showed that a Nesterov-based outer optimizer improves performance. Recently \citet{DES-LOC} improved the communication efficiency of \localadam by decoupling parameter and momentum sync frequencies. However, these methods use single-timescale optimizers with small $\beta_1$ values ill-suited to low communication frequencies due to momentum decay. While naively increasing momentum often harms task performance, recent optimizers that track gradients across multiple timescales have shown significant benefits. \qhm~\citep{QHM} decouples momentum decay from gradient weight, while \aggmo~\citep{AggMo} averages multiple velocity vectors for stability. Building on this, \ademamix~\citep{AdEMAMix} mixes fast and slow momenta to accelerate convergence, demonstrating that slow momentum acts as memory, reducing forgetting in LLMs. Fur further related work see \cref{app:sec:extended_related_work}.


\section{Conclusion}
A persistent challenge in distributed training has been the performance gap between fully-synchronous and communication-efficient optimizers. We identify one potential cause for this gap: the rapid decay of momentum in standard optimizers is temporally mismatched with the long intervals inherent to infrequent communication, leading to unstable update directions. We address this with \method, a multi-timescale optimizer that maintains a stable, long-term (high-$\beta$) update direction that persists across communication rounds. Our theory shows that momenta with higher $\beta$ are less sensitive to synchronization frequency. Furthermore, our experiments on large language models demonstrate that this approach closes the performance gap with \ddp and outperforms prior communication-efficient methods. This is achieved by using the slow momentum to maintain a stable, shared optimization trajectory across workers. These findings establish that managing momentum timescales is a critical factor for performant distributed training, opening new avenues for research into dynamic timescale modulation and integration with compression. Ultimately, this work provides a robust and practical path forward for scaling foundation model training in communication-constrained environments, for cross-datacenter training, or across wide geographic areas.


\bibliography{iclr2026_conference}
\bibliographystyle{iclr2026_conference}

\clearpage
\newpage

\appendix
\setcounter{parttocdepth}{2}
\mtcsettitle{parttoc}{Table of Contents}

\addcontentsline{toc}{section}{Appendix} %

\renewcommand \thepart{} %
\renewcommand \partname{}
\part{\Large{\centerline{Appendix}}}

\parttoc

\clearpage

\section{Limitations}\label{app:sec:limitations}
\textbf{Limitations.} First, our empirical validation is limited to models up to $720$M parameters.  Second, our experiments use \adopt instead of \adam, this avoids the need to tune the $\beta_2$ of the second momentum, simplifying experimental design. Third, we preferred a detailed investigation of the training dynamics of the highly memory and communication-efficient \method~($N=1$) over increasing the number of momenta, which brings diminishing returns~\citep{AggMo,QHM}.

\section{Experimental Details}~\label{app:exp_details}

Here we provide additional experimental details complementing those in \cref{sec:exp_setup}, including: a) model architecture details and the model parameterization (\cref{app:subsec:architecture_details}), b) our hyperparameter sweep procedure to select optimizer-specific settings (\cref{app:subsec:hyperparameter_sweeping_procedure}), and c) the results of our tuning sweeps for \method.
\subsection{Architecture Details and Parametrization}\label{app:subsec:architecture_details}

\begin{table}[H]
\caption{Model architecture and training hyperparameters. Architectural parameters include the number of transformer blocks (\#Blocks), attention heads (\#Heads), embedding dimension ($d_\mathrm{model}$), vocabulary size ($|\mathcal{V}|$), and feedforward expansion ratio (Exp.~Ratio). Key training parameters are the global batch size ($|\mathcal{B}_{\mathrm{G}}|$) and the total number of training steps ($T$). All models use RoPE positional embeddings~\citep{RopeEmbeddings}, the \texttt{SiLU} activation function, norm-based gradient clipping with a bound of $\rho$, and are initialized with a typical~\citep{BenchmarkingOptimizersLLM,CompleteP} $\sigma=0.02$. Sequence length is standard for models at these scales.}\label{tab:model_architectures} 
\centering
\resizebox{\textwidth}{!}{%
\begin{tabular}{@{}lcccccccccccc@{}}
\toprule
\textbf{Model Size} &
  \textbf{Blocks} &
  \textbf{$\boldsymbol{d_{\mathrm{model}}}$} &
  \textbf{$\mathbf{|\mathcal{V}|}$} &
  \textbf{\#Heads} &
  \textbf{Exp.$\sim$Ratio} &
  \textbf{ROPE $\theta$} &
  \textbf{ACT} &
  \textbf{Init $\sigma$} &
  \textbf{$\rho$} &
  \textbf{Seq Len} &
  \textbf{$\mathbf{|\mathcal{B}_{\mathrm{G}}|}$} &
  \textbf{$\mathbf{T}$} \\ \midrule
$16$M  & $4$  & $256$  & $50$K & $4$  & $4$ & $10000$ & \texttt{SiLU} & $0.02$ & $1.0$ & $2048$ & $64$   & $4608,38912$ \\
$125$M & $12$ & $768$  & $50$K & $12$ & $4$ & $10000$ & \texttt{SiLU}  & $0.02$ & $1.0$ & $2048$ & $256$  & $38912$      \\
$720$M & $12$ & $2048$ & $50$K & $16$ & $4$ & $10000$ & \texttt{SiLU}  & $0.02$ & $1.0$ & $2048$ & $512$ & $38912$      \\ \bottomrule
\end{tabular}%
}
\end{table}
\Cref{tab:model_architectures} summarizes the architectural details of our models, which follow established practices for large language models at their respective scales. To improve training stability and final performance, we adopt two key modifications. First, following the recommendations of \citet{PeriLayerNorm}, we use a Peri-LayerNorm transformer structure instead of pre-norm. Second, we use the \texttt{CompleteP}~\citep{CompleteP} parametrization with $\alpha=1.0$, which enables the effective transfer of optimizer hyperparameters from a small model to its larger-scale counterparts in a one-shot manner. This property allows us to perform comprehensive hyperparameter sweeps on our smallest model size and reserve computationally expensive scaling experiments for direct comparisons against baselines. 

We set batch sizes and training durations following recent best practices~\citep{HowDoesBatchSizeScaleInPreTraining}. For the smallest model size, the initial batch size is determined using the noise-scale estimator for the critical batch size~\citep{LargeBatchTraining} and then doubled until the efficiency deviates from a linear trend by 20\%. For our $125$M and $720$M models we follow the batch size recommendations from \citet{BenchmarkingOptimizersLLM}.  Training durations are set as multiples of the compute-optimal token budget~\citep{TrainingComputeOptimalLLMs}: for the $16$M model, we tune using $\approx 2\times$ this budget and run baseline comparisons at $\approx 16\times$; for the $125$M model, we use $\approx 8\times$; and for the $720$M model, we use $\approx 2.83\times$. We chose the $720$M model size as a good balance between scale and computational efficiency following \citet{BenchmarkingOptimizersLLM}.

All models are trained using the warmup-stable-decay (\texttt{WSD}) learning rate schedule~\citep{BeyondFixedTrainingDuration}, with warmup and decay periods selected based on established recommendations~\citep{HowDoesBatchSizeScaleInPreTraining,BeyondFixedTrainingDuration,SmolLM2,BenchmarkingOptimizersLLM}. For the $16$M model tuning runs, which last $4608$ steps, the warmup period is set to $T_{\mathrm{WARM}} = 512$ steps. For all longer training runs and baseline comparisons, we use the industry-standard warmup of $T_{\mathrm{WARM}} = 2048$ steps. We use a cooldown period equal to the warmup period in all cases, using 1-sqrt cooldown~\citep{BeyondFixedTrainingDuration}.

\subsection{Optimizer Hyperparameter Sweeping Procedure}\label{app:subsec:hyperparameter_sweeping_procedure}
Our tuning procedure is designed to ensure that both our method and the baselines are evaluated under their optimal DDP configurations, providing a fair comparison. Given that previous work has shown that the learning rate (LR) tends to transfer effectively between DDP and distributed settings~\citep{DES-LOC}, we first tune all parameters to achieve the best possible performance under DDP and then transfer these settings to \method.
Unlike methods such as \ademamix that use schedulers for optimizer parameters, we employ a simple switch from a base optimizer (e.g., \adopt) to its multi-timescale variant at the end of the warmup period. This necessitates a two-phase LR tuning process to ensure identical starting conditions for both optimizers:
\begin{enumerate}
\item \textbf{Phase 1: Base Optimizer Tuning.} We first tune the learning rate for the base optimizer over the entire training run to achieve the lowest final perplexity. This ensures the baseline itself is as strong as possible.
\item \textbf{Phase 2: \method /Quasi-hyperbolic Tuning.} Using the model state from the end of the base optimizer's warmup, we then tune the learning rate for the post-switch phase of \method and of its \ddp analogue. With a \texttt{WSD} scheduler, this corresponds to tuning the LR for the constant "stable" portion of training and for the cooldown.
\end{enumerate}
While more complex scheduling manipulations might yield further gains for \method, this two-phase approach provides the cleanest methodology for comparison. For every combination of momentum decay rates ($\beta$'s) and convex coefficients ($\omega$'s) used by \method, we independently perform this tuning procedure. For \adopt the LR sweeps in both phases search over values between $2^{-10}$ and $2^{-6}$ using powers of two, with the search grid refined by manually adding half-power steps (e.g., $2^{-8.5}, 2^{-7.5}, 2^{-6.5}$) around the optimal value.

\subsection{Optimizer Tuning Results}

First, our tuning of \adopt for \ddp revealed an optimal lr $\eta^\ast = 2^{-8}$. We now present the results of our tuning for the post-warmup lr for \methodadopt with $N=1$ first momenta in \cref{fig:qhm_tuning_results}. 

\begin{figure}[htb]
    \centering
    \includegraphics[width=\linewidth]{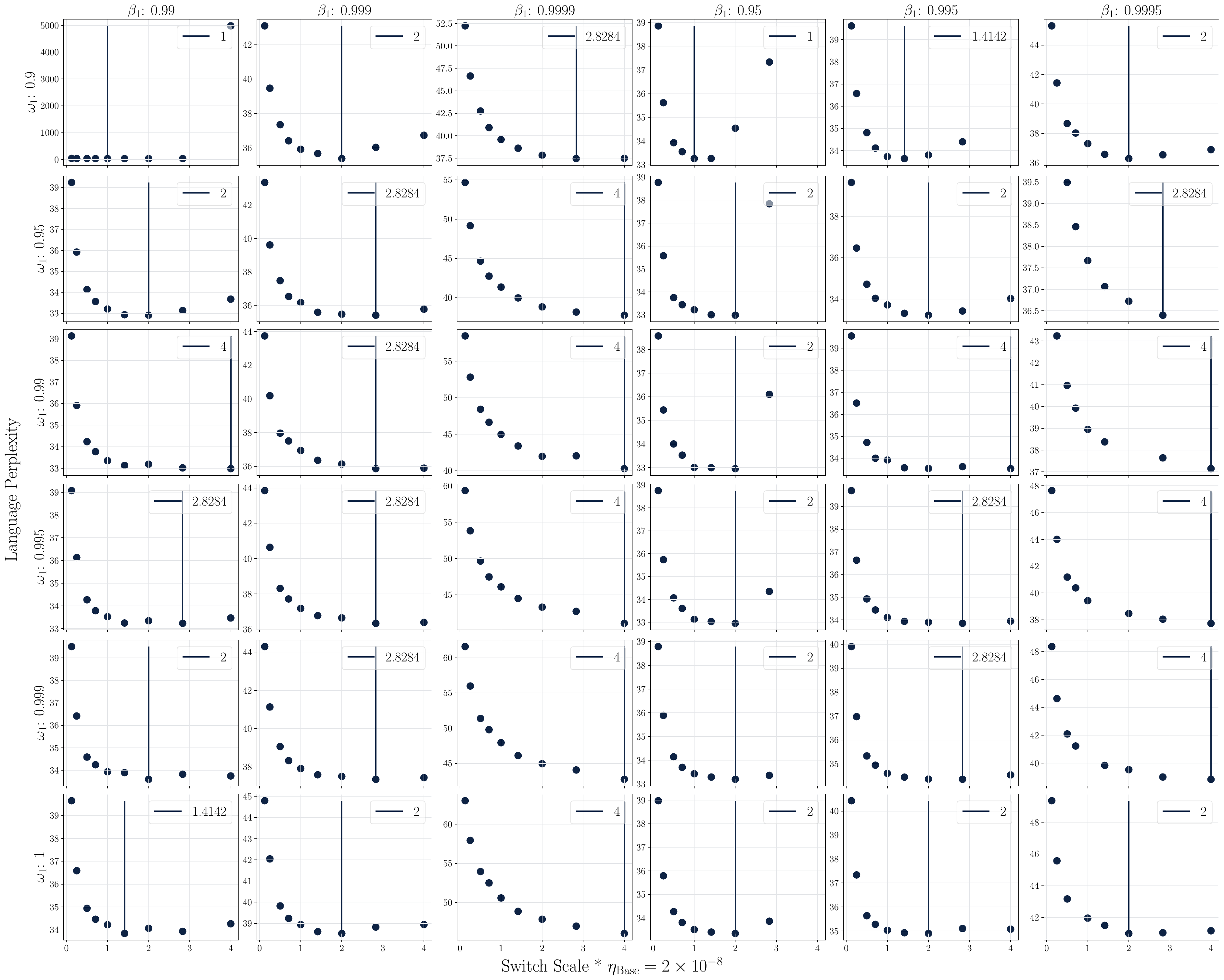}
    \caption{Visualizing the learning rate sweeps for different \method configurations. Each subplot shows the final perplexity for a given convex coefficient ($\omega$) and momentum decay ($\beta_1$), where $\beta_2=0.999$ was kept constant. The sweep demonstrates that the optimal learning rate and final performance are highly dependent on the choice of these internal hyperparameters, with $\beta_1 \in [0.995,0.999]$ and $\omega \in [0.9, 0.99]$ performing best for these short tuning experiments. The vertical line in each subplot marks the best-performing lr for that configuration.Switch scale referes to the multiple of the base learning rate that we select, the chosen learning rate can be computed via multiplication with $\eta_{\mathrm{BASE}}.$}
    \label{fig:qhm_tuning_results}
\end{figure}

A clear trend emerges from our results: methods with higher momentum decay rates ($\beta$s) or higher weights ($\omega$s) ascribed to the slow-moving momenta can tolerate significantly higher learning rates than standard momentum methods. This finding is in strong agreement with the previous findings of \citet{AggMo}, who similarly found that \aggmo can effectively utilize learning rates that are orders of magnitude higher than those suitable for classical momentum. 

\takeawaybox[Takeaway:]{%
Multi-timescale optimizers that emphasize slow-moving momenta (via high $\beta$ or $\omega$ values) are not only more stable but can also leverage much higher learning rates, enabling faster convergence than their single-timescale counterparts.
}

\newpage
\section{Deterministic Optimizer-specific Variants of \method}
\addcontentsline{toc}{subsection}{\methodadopt}
\begin{algorithm}[h]
\caption{\methodadopt}
\label{alg:method_adopt}
\footnotesize
\begin{onehalfspace}
\begin{algorithmic}[1]
 \Require \textbf{Model tensors, hyper-parameters} \\
      \quad $x_0 \in \mathbb{R}^{d}$, $\{\bar{u}^{j}_{-1}\}_{j=1}^{N} \in (\mathbb{R}^d)^N$, $\bar{v}_{-1} \in \mathbb{R}^d$ — initial params, $N$ first momenta, second momentum\\
      \quad $\{\beta_{1,j}\}_{j=1}^N, \beta_2 \in [0,1)$ — decay rates for each momentum state \\
      \quad \textcolor{purple}{$\{\omega_j\}_{j=1}^N \in [0,1]$} — convex combination coefficients for first momenta, $\sum_{j=1}^N \omega_j \leq 1.0$ \\
      \quad $\{c_t\}_{t=0}^{T-1}, \{\eta_t\}_{t=0}^{T-1}$ — clipping and learning-rate schedules\\
      \quad $T,M \in \mathbb{N}_{+}$ — total optimization steps and number of workers\\
      \quad $K_x, \{K_j\}_{j=1}^{N}, K_v \in (\mathbb{N}_{+})^{N+2}$ — communication periods for parameters and states \\
      \quad $\texttt{OuterOpt}:\mathbb{R}^d\to\mathbb{R}^d$ — update params using an outer optimizer, averaging by default 
 \Ensure $x_T,\;\{u_{T-1}^{j}\}_{j=1}^{N},\;v_{T-1}$

 \State \textbf{for each worker} $m$: initialize $x_0^m, \{u_{-1}^{j,m}\}, v_{-1}^m$
 \For{$t = 0,\dots,T-1$}
     \ForAll{workers $m=0,\dots,M-1$ \textbf{in parallel}}
         \State $g_t^m \gets \nabla F(x_t^m;\xi_t^m)$ \hfill\textcolor{gray}{\scriptsize stochastic gradient}
         \State $v^m_t \gets \beta_2\bar{v}_{t-1} + (1-\beta_2)(g^m_t)^2$ \hfill\textcolor{gray}{\scriptsize update second momentum with raw gradient}
         \State $\bar{v}_{t} \gets$  \textbf{if} $(t \bmod K_v = 0)$ \textbf{then} $\mathbb{E}_m[v_{t}^{m}]$ \textbf{else} $v_{t}^{m}$ \hfill\textcolor{gray}{\scriptsize sync $v$ every $K_v$ steps}
         
         \State $\hat{g}_t^m \gets \frac{g_t^m}{\sqrt{v_t^m}+\epsilon}$ \hfill\textcolor{gray}{\scriptsize normalize gradient (ADOPT core step)}
         \State $\tilde{g}_t^m \gets \clip(\hat{g}_t^m, c_t)$ \hfill\textcolor{gray}{\scriptsize clip the normalized gradient}

         \For{$j = 1$ \textbf{to} $N$} \hfill\textcolor{gray}{\scriptsize update N first momenta}
             \State \textcolor{purple}{$u_t^{j,m} \gets \beta_{1,j} \bar{u}_{t-1}^{j} + (1-\beta_{1,j}) \tilde{g}_t^m$} \hfill\textcolor{purple}{\scriptsize use clipped, normalized gradient}
             \State \textcolor{purple}{$\bar{u}_{t}^{j} \gets$ \textbf{if} $(t \bmod K_j = 0)$ \textbf{then} $\mathbb{E}_m[u_{t}^{j,m}]$ \textbf{else} $u_{t}^{j,m}$ } \hfill\textcolor{purple}{\scriptsize sync $u^j$ every $K_j$ steps}
         \EndFor

         \State \textcolor{purple}{$\Delta_t^m \gets (1 - \sum_{j=1}^N \omega_j) \tilde{g}_t^m + \sum_{j=1}^N \omega_j u_t^{j,m}$} \hfill\textcolor{purple}{\scriptsize form combined update direction}
         
         \State $x_{t+1}^m \gets \bar{x}_t - \eta_t \Delta_t^m$ \hfill\textcolor{gray}{\scriptsize apply combined update}
         \State $\bar{x}_{t+1} \gets$  \textbf{if} $((t+1) \bmod K_x = 0) $ \textbf{then} \texttt{OuterOpt}($\mathbb{E}_m[x_{t+1}^{m}]$) \textbf{else} $x_{t+1}^{m}$ \hfill\textcolor{gray}{\scriptsize sync $x$ every $K_x$ steps}
     \EndFor
 \EndFor
\end{algorithmic}
\end{onehalfspace}
\end{algorithm}
\addcontentsline{toc}{subsection}{\methodsgdm}
\begin{algorithm}[h]
\caption{\methodsgdm}
\label{alg:method_sgdm_non_prob}
\footnotesize
\begin{onehalfspace}
\begin{algorithmic}[1]
 \Require \textbf{Model tensors, hyper-parameters} \\
      \quad $x_0 \in \mathbb{R}^{d}$, $\{\bar{u}^{j}_{-1}\}_{j=1}^{N} \in (\mathbb{R}^d)^N$ — initial params, $N$ first momenta\\
      \quad $\{\beta_{1,j}\}_{j=1}^N \in [0,1)$ — decay rates for each momentum state \\
      \quad \textcolor{purple}{$\{\omega_j\}_{j=1}^N \in [0,1]$} — convex combination coefficients for first momenta, $\sum_{j=1}^N \omega_j \leq 1.0$ \\
      \quad $\rho \in \mathbb{R}_{+}$, $\{\eta_t\}_{t=0}^{T-1}$ — clipping radius, learning-rate schedule\\
      \quad $T,M \in \mathbb{N}_{+}$ — total optimization steps and number of workers\\
      \quad $K_x, \{K_j\}_{j=1}^{N} \in (\mathbb{N}_{+})^{N+1}$ — communication periods for parameters and states \\
      \quad $\texttt{OuterOpt}:\mathbb{R}^d\to\mathbb{R}^d$ — update params using an outer optimizer, averaging by default 

 \Ensure $x_T,\;\{u_{T-1}^{j}\}_{j=1}^{N}$

 \State \textbf{for each worker} $m$: initialize $x_0^m, \{u_{-1}^{j,m}\}$
 \For{$t = 0,\dots,T-1$}
     \ForAll{workers $m=0,\dots,M-1$ \textbf{in parallel}}
         \State $\hat{g}_t^m \gets \clip(\nabla F(x_t^m;\xi_t^m),\rho)$ \hfill\textcolor{gray}{\scriptsize clipped stochastic gradient}

         \For{$j = 1$ \textbf{to} $N$} \hfill\textcolor{gray}{\scriptsize update N first momenta}
             \State \textcolor{purple}{$u_t^{j,m} \gets \beta_{1,j} \bar{u}_{t-1}^{j} + (1-\beta_{1,j}) \hat{g}_t^m$}
             \State \textcolor{purple}{$\bar{u}_{t}^{j} \gets$ \textbf{if} $(t \bmod K_j = 0)$ \textbf{then} $\mathbb{E}_m[u_{t}^{j,m}]$ \textbf{else} $u_{t}^{j,m}$ } \hfill\textcolor{purple}{\scriptsize sync $u^j$ every $K_j$ steps} 
         \EndFor

         \State \textcolor{purple}{$\Delta_t^m \gets (1 - \sum_{j=1}^N \omega_j) \hat{g}_t^m + \sum_{j=1}^N \omega_j u_t^{j,m}$} \hfill\textcolor{purple}{\scriptsize form combined update direction (unnormalized)}
         
         \State $x_{t+1}^m \gets \bar{x}_t - \eta_t \Delta_t^m$
         \State $\bar{x}_{t+1} \gets$  \textbf{if} $((t+1) \bmod K_x = 0) $ \textbf{then} \texttt{OuterOpt}($\mathbb{E}_m[x_{t+1}^{m}]$) \textbf{else} $x_{t+1}^{m}$ \hfill\textcolor{gray}{\scriptsize sync $x$ every $K_x$ steps}
     \EndFor
 \EndFor
\end{algorithmic}
\end{onehalfspace}
\end{algorithm}

\FloatBarrier
\section{Convergence Analysis of \methodsgdm}\label{app:proof-sgdm}

\begin{algorithm}[htb]
\caption{\methodsgdm, probabilistic variant}
\label{alg:method_sgdm_prob}
\footnotesize
\begin{onehalfspace}
\begin{algorithmic}[1]
 \Require \textbf{Model tensors, hyper-parameters} \\
        \quad $x_0 \in \mathbb{R}^{d}$, $\{u^{j}_{-1}\}_{j=1}^{N} \in (\mathbb{R}^d)^N$ — initial parameters, $N$ first momenta\\
        \quad $\{\beta_{j}\}_{j=1}^N \in [0,1)$ — decay rates for each momentum state \\
        \quad \textcolor{purple}{$\{\omega_j\}_{j=1}^N \in [0,1]$} — convex combination of non-negative coefficients for first momenta, $\sum_{j=1}^N w_j = 1$ \\
        \quad $\{\eta_t\}_{t=0}^{T-1}$ — learning-rate schedule\\
        \quad $T,M \in \mathbb{N}_{+}$ — total optimization steps and number of workers\\
        \quad \textcolor{purple}{$p_x = \frac{1}{K_x}, \{p_j = \frac{1}{K_j}\}_{j=1}^{N} \in [0,1]^{N+1}$} — communication periods/probabilities for parameters and states

 \Ensure $x_T,\;\{u_{T-1}^{j}\}_{j=1}^{N}$

 \State \textbf{for each worker} $m$: initialize $x_0^m, \{u_{-1}^{j,m}\}$
 \For{$t = 0,\dots,T-1$}
    \ForAll{workers $m=0,\dots,M-1$ \textbf{in parallel}}
        \State $g_t^m \gets \nabla F_m(x_t^m;\xi_t^m)$
           \hfill\textcolor{gray}{\scriptsize stochastic gradient}

        \For{$j = 1$ \textbf{to} $N$} \hfill\textcolor{gray}{\scriptsize update N first momenta}
            \State $u_t^{j,m} \gets
                \begin{cases}
                \E_m[\beta_j u_{t-1}^{j,m} + (1-\beta_j)g_{t}^m], & \textcolor{purple}{\text{with probability } p_j} \\ 
                \beta_j u_{t-1}^{j,m} + (1-\beta_j)g_{t}^m, & \text{with probability } 1-p_j 
                \end{cases}$ \hfill\textcolor{purple}{\scriptsize sync $u$}
        \EndFor

        \State $\Delta_t^m \gets \sum_{j=1}^N \omega_j u_t^{j,m}$ \hfill\textcolor{purple}{\scriptsize form combined update direction}
        \State $x_{t+1}^{m} \gets
            \begin{cases}
            \E_m[x_{t}^m - \eta_t \Delta_{t}^m], & \textcolor{purple}{\text{with probability } p_x} \\ 
            x_{t}^m - \eta_t \Delta_{t}^m, & \text{with probability } 1-p_x 
            \end{cases}$ \hfill\textcolor{gray}{\scriptsize sync $x$}
        \EndFor
 \EndFor
\end{algorithmic}
\end{onehalfspace}
\end{algorithm}

In order to facilitate the technical presentation, we model synchronization frequencies by assigning probabilities to each averaging event. For example, the parameters $x_t^m$ are synchronized with the probability $p_x = \frac{1}{K_x}$, which is statistically equivalent to performing the averaging in every $\frac{1}{p_x} = K_x$ iteration. Similarly, momentum $u_t^{j,m}$ synchronization happens with probability $p_j = \frac{1}{K_j}$, which can differ from $p_x$. Note that QHM structure is included since we can choose $\beta_1=0$ and get $u_t^{1,m} = g_t^m$.

\underline{Auxiliary notation.} Let $\E_m$ and $\E_j$ be the averaging operators with weights $\frac{1}{M}$ across $M$ workers and $\omega_j$ across $N$ momenta.
\begin{eqnarray*}
u_t^j &\eqdef& \E_m[u_t^{j,m}] = \beta_j u_{t-1}^j + (1-\beta_j)g_t, \text{where } g_t = \E_m[g_t^m] \\
x_{t+1}^{j,m} &\eqdef&
            \begin{cases}
            \E_m[x_{t}^{j,m} - \eta u_{t}^{j,m}], & \text{with probability } p_x \\ 
            x_{t}^{j,m} - \eta u_{t}^{j,m}, & \text{with probability } 1-p_x
            \end{cases} \\
x_{t+1}^j &\eqdef& \E_m[x_{t+1}^{j,m}] = x_t^j - \eta u_t^j, \quad x_{t+1}^m = \E_j [x_{t+1}^{j,m}] = \text{ (line 14) }.
\end{eqnarray*}
For the sake of notation, we also let $u_t^m = \Delta_t^m = \E_j[u_t^{j,m}],\, u_t = \E_m[u_t^m],\, x_t = \E_m[x_t^m]$ in the upcoming derivations.

\underline{Step 1 (virtual iterates).} Letting $x_{-1}^j = x_0^j = x_0$, define the global virtual iterations as follows
$$
z_t^j \eqdef \frac{1}{1-\beta_j} x_t^j - \frac{\beta_j}{1-\beta_j} x_{t-1}^j, \quad \text{and} \quad z_t \eqdef \E_j [z_t^j] \quad \text{for } t\ge0.
$$
The key property of this virtual iterates we are going to exploit in the next steps is that they follow averaged gradients, namely for any $t\ge0$ we have
\begin{eqnarray*}
z_{t+1} - z_t
&=& \E_j [z_{t+1}^j - z_t^j] \\
&=& \E_j  \left[\left(\frac{1}{1-\beta_j} x_{t+1}^j - \frac{\beta_j}{1-\beta_j} x_{t}^j\right) - \left(\frac{1}{1-\beta_j} x_t^j - \frac{\beta_j}{1-\beta_j} x_{t-1}^j\right) \right] \\
&=& \E_j \left[\frac{1}{1-\beta_j} (x_{t+1}^j - x_{t}^j) - \frac{\beta_j}{1-\beta_j} (x_{t}^j - x_{t-1}^j) \right] \\
&=& \E_j \left[\frac{1}{1-\beta_j}(-\eta u_t^j) - \frac{\beta_j}{1-\beta_j} (-\eta u_{t-1}^j) \right] \\
&=& \E_j \left[\frac{-\eta}{1-\beta_j}(u_t^j - \beta_j u_{t-1}^j) \right] 
= \E_j[-\eta g_t] = -\eta g_t.\\
\end{eqnarray*}

\underline{Step 2 (smoothness over virtual iterates).} Then we apply smoothness of the global loss function $f$ over these global virtual iterates.
\begin{eqnarray*}
f(z_{t+1})
&\le& f(z_t) + \langle \nabla f(z_t), z_{t+1} - z_t \rangle + \frac{L}{2}\|z_{t+1} - z_t\|^2 \\
&=& f(z_t)
+ \underbrace{\langle \nabla f(x_t), z_{t+1} - z_t \rangle}_{I}
+ \underbrace{\langle \nabla f(z_t) - \nabla f(x_t), z_{t+1} - z_t \rangle}_{II}
+ \underbrace{\frac{L}{2}\|z_{t+1} - z_t\|^2}_{III}.
\end{eqnarray*}

In the next step, we separately bound each term appearing in the above bound.

\underline{Step 3a (one step progress).} Bounding term I.
\begin{eqnarray*}
&& \E{\langle \nabla f(x_t), z_{t+1} - z_t \rangle} \\
&=& -\eta \E{\left\langle \nabla f(x_t), \frac{1}{M}\sum_{m=1}^M g_t^m \right\rangle}
= -\eta \E{\left\langle \nabla f(x_t), \frac{1}{M}\sum_{m=1}^M \nabla f_m(x_t^m) \right\rangle} \\
&=& -\frac{\eta}{2}\E{\|\nabla f(x_t)\|^2} - \frac{\eta}{2}\E{\left\|\frac{1}{M}\sum_{m=1}^M \nabla f_m(x_t^m) \right\|^2} + \frac{\eta}{2}\E{\left\| \nabla f(x_t) - \frac{1}{M}\sum_{m=1}^M \nabla f_m(x_t^m) \right\|^2} \\
&=& -\frac{\eta}{2}\E{\|\nabla f(x_t)\|^2} - \frac{\eta}{2}\E{\left\|\frac{1}{M}\sum_{m=1}^M \nabla f_m(x_t^m) \right\|^2} + \frac{\eta}{2}\E{\left\| \frac{1}{M}\sum_{m=1}^M \nabla f_m(x_t) - \nabla f_m(x_t^m) \right\|^2} \\
&\le& -\frac{\eta}{2}\E{\|\nabla f(x_t)\|^2} - \frac{\eta}{2}\E{\left\|\frac{1}{M}\sum_{m=1}^M \nabla f_m(x_t^m) \right\|^2} + \frac{\eta}{2M}\sum_{m=1}^M\E{\left\| \nabla f_m(x_t) - \nabla f_m(x_t^m) \right\|^2} \\
&\le& -\frac{\eta}{2}\E{\|\nabla f(x_t)\|^2} - \frac{\eta}{2}\E{\left\|\frac{1}{M}\sum_{m=1}^M \nabla f_m(x_t^m) \right\|^2} + \frac{\eta L^2}{2M}\sum_{m=1}^M \underbrace{\E{\left\| x_t - x_t^m \right\|^2}}_{\rm Lemma\;\ref{lem:x-xm}}.
\end{eqnarray*}

\underline{Step 3b (one step progress).} Bounding term II.
\begin{eqnarray*}
\E{\langle \nabla f(z_t) - \nabla f(x_t), z_{t+1} - z_t \rangle}
&=& -\eta \E{\left\langle \nabla f(z_t) - \nabla f(x_t), \frac{1}{M}\sum_{m=1}^M \nabla f_m(x_t^m) \right\rangle} \\
&\le& \frac{\eta\rho}{2} \E{\| \nabla f(z_t) - \nabla f(x_t)\|^2} + \frac{\eta}{2\rho} \E{\left\|\frac{1}{M}\sum_{m=1}^M \nabla f_m(x_t^m) \right\|^2} \\
&\le& \frac{\eta\rho L^2}{2} \underbrace{\E{\| z_t - x_t \|^2}}_{\rm Lemma\;\ref{lem:z-x}} + \frac{\eta}{2\rho} \E{\left\|\frac{1}{M}\sum_{m=1}^M \nabla f_m(x_t^m) \right\|^2}.
\end{eqnarray*}

\underline{Step 3c (one step progress).} Bounding term III.
\begin{eqnarray*}
\frac{L}{2}\E{\|z_{t+1} - z_t\|^2}
&=& \frac{\eta^2L}{2}\E{\left\| \frac{1}{M}\sum_{m=1}^M g_t^m \right\|^2} \\
&=& \frac{\eta^2L}{2}\E{\left\| \frac{1}{M}\sum_{m=1}^M g_t^m - \nabla f_m(x_t^m) \right\|^2} + \frac{\eta^2L}{2}\E{\left\| \frac{1}{M}\sum_{m=1}^M \nabla f_m(x_t^m) \right\|^2} \\
&=& \frac{\eta^2L}{2M^2}\sum_{m=1}^M\E{\left\| g_t^m - \nabla f_m(x_t^m) \right\|^2} + \frac{\eta^2L}{2}\E{\left\| \frac{1}{M}\sum_{m=1}^M \nabla f_m(x_t^m) \right\|^2} \\
&\le& \frac{\eta^2L}{2M}\sigma^2 + \frac{\eta^2L}{2}\E{\left\| \frac{1}{M}\sum_{m=1}^M \nabla f_m(x_t^m) \right\|^2}.
\end{eqnarray*}

\underline{Step 3abc (one step progress).} Combining previous bounds.
\begin{eqnarray*}
\E{f(z_{t+1})} - \E{f(z_t)}
&\le& \E \underbrace{\langle \nabla f(x_t), z_{t+1} - z_t \rangle}_{I}
+ \E \underbrace{\langle \nabla f(z_t) - \nabla f(x_t), z_{t+1} - z_t \rangle}_{II}
+ \E \underbrace{\frac{L}{2}\|z_{t+1} - z_t\|^2}_{III} \\
&\le& -\frac{\eta}{2}\E{\|\nabla f(x_t)\|^2} - \frac{\eta}{2}\E{\left\|\frac{1}{M}\sum_{m=1}^M \nabla f_m(x_t^m) \right\|^2} + \frac{\eta L^2}{2M}\sum_{m=1}^M \underbrace{\E{\left\| x_t - x_t^m \right\|^2}}_{\rm Lemma\;\ref{lem:x-xm}} \\
&& +\; \frac{\eta\rho L^2}{2} \underbrace{\E{\| z_t - x_t \|^2}}_{\rm Lemma\;\ref{lem:z-x}} + \frac{\eta}{2\rho} \E{\left\|\frac{1}{M}\sum_{m=1}^M \nabla f_m(x_t^m) \right\|^2} \\
&& +\; \frac{\eta^2L}{2K}\sigma^2 + \frac{\eta^2L}{2}\E{\left\| \frac{1}{M}\sum_{m=1}^M \nabla f_m(x_t^m) \right\|^2} \\
&\le& -\frac{\eta}{2}\E{\|\nabla f(x_t)\|^2}
      - \frac{\eta}{2} \left( 1 - \frac{1}{\rho} - \eta L \right)\E{\left\|\frac{1}{M}\sum_{m=1}^M \nabla f_m(x_t^m) \right\|^2} \\
&& +\; \frac{\eta\rho L^2}{2} \underbrace{\E{\| z_t - x_t \|^2}}_{\rm Lemma\;\ref{lem:z-x}} + \frac{\eta L^2}{2M}\sum_{m=1}^M \underbrace{\E{\left\| x_t - x_t^m \right\|^2}}_{\rm Lemma\;\ref{lem:x-xm}} + \frac{\eta^2L}{2M}\sigma^2.
\end{eqnarray*}

\underline{Step 4 (final).} Now we average over the iterates and apply the bounds derived in Lemmas \ref{lem:z-x} and \ref{lem:x-xm}.
\begin{eqnarray*}
\frac{\E[f(z_T) - f(z_0)]}{T}
&=& \frac{1}{T}\sum_{t=0}^{T-1} \E[f(z_{t+1}) - f(z_t)] \\
&\le& -\frac{\eta}{2T}\sum_{t=0}^{T-1}\E{\|\nabla f(x_t)\|^2}
      - \frac{\eta}{2} \left( 1 - \frac{1}{\rho} - \eta L \right) \frac{1}{T}\sum_{t=0}^{T-1}\E{\left\|\frac{1}{M}\sum_{m=1}^M \nabla f_m(x_t^m) \right\|^2} \\
&& +\; \frac{\eta\rho L^2}{2} \underbrace{\frac{1}{T}\sum_{t=0}^{T-1}\E{\| z_t - x_t \|^2}}_{\rm Lemma\;1}
+ \frac{\eta L^2}{2} \underbrace{\frac{1}{TM}\sum_{t=0}^{T-1}\sum_{m=1}^M\E{\left\| x_t - x_t^m \right\|^2}}_{\rm Lemma\;2} + \frac{\eta^2L}{2M}\sigma^2 \\
&\le& -\frac{\eta}{2T}\sum_{t=0}^{T-1}\E{\|\nabla f(x_t)\|^2}
      - \frac{\eta}{2} \left( 1 - \frac{1}{\rho} - \eta L \right) \frac{1}{T}\sum_{t=0}^{T-1}\E{\left\|\frac{1}{M}\sum_{m=1}^M \nabla f_m(x_t^m) \right\|^2}
       + \frac{\eta^2L}{2M}\sigma^2 \\
&& +\; \frac{\eta\rho L^2}{2} \left( \frac{\eta^2\beta_{\omega}^2}{M}\sigma^2 + \eta^2\beta_{\omega}^2 \frac{1}{T}\sum_{\tau=0}^{T-1} \E\left\| \frac{1}{M}\sum_{m=1}^M \nabla f_m(x_{\tau}^m) \right\|^2 \right) \\
&& +\;   \frac{\eta L^2}{2} \left( 12\eta^2 (B^2-1) \psi \cdot \frac{1}{T}\sum_{t=0}^{T-1}\E\|\nabla f(\theta^{t})\|^2
+ 4\eta^2\psi(\sigma^2 + 3G^2) \right) \\
&\le& -\frac{\eta}{2}\left(1 - 12\eta^2L^2(B^2-1)\psi\right) \frac{1}{T}\sum_{t=0}^{T-1}\E{\|\nabla f(x_t)\|^2} \\
&& -\; \frac{\eta}{2} \left( 1 - \frac{1}{\rho} - \eta L - \eta^2\beta_{\omega}^2 \rho L^2 \right) \frac{1}{T}\sum_{t=0}^{T-1}\E{\left\|\frac{1}{M}\sum_{m=1}^M \nabla f_m(x_t^m) \right\|^2} \\
&& +\; \frac{\eta^2L}{2M}\sigma^2 
+ \frac{\eta^3\rho L^2 \beta_{\omega}^2}{2M}\sigma^2 
+ 2\eta^3 L^2\psi(\sigma^2 + 3G^2).
\end{eqnarray*}
Next, we choose $\rho=2$ and step size $\eta$ such that
\begin{eqnarray*}
12\eta^2L^2(B^2-1)\psi \le \frac{1}{2} &\iff& \text{to bound the first term} \\
\eta L + 2\eta^2\beta_{\omega}^2 L^2 \le \frac{1}{2} &\iff& \text{to bound the second term} \\
12\eta^2 L^2 \psi \le \frac{1}{2} &\iff& \text{from Lemma~\ref{lem:x-xm}}
\end{eqnarray*}
Notice that
$$
\eta_0 \eqdef \left( 4L \max\left(\beta_{\omega}, 6\sqrt{\psi\max(1,B^2-1)} \right) \right)^{-1}
$$
satisfies all three bounds. Then, with any $\eta\le\eta_0$ we get
\begin{eqnarray*}
\frac{\E[f(z_T) - f(z_0)]}{T}
&\le& -\frac{\eta}{4T}\sum_{t=0}^{T-1}\E{\|\nabla f(x_t)\|^2} \\
&& +\; \frac{\eta^2L}{2M}\sigma^2 
+ \frac{\eta^3\rho L^2 \beta_{\omega}^2}{2M}\sigma^2 
+ 2\eta^3 L^2\psi(\sigma^2 + 3G^2).
\end{eqnarray*}
Noticing that $z_0=x_0$ and $f^* \le f(z_T)$, we have
\begin{eqnarray*}
\frac{1}{T}\sum_{t=0}^{T-1}\E{\|\nabla f(x_t)\|^2}
\le \frac{4(f(x_0) - f^*)}{\eta T}
+ \frac{2\eta L}{M}\sigma^2 
+ \frac{4\eta^2 L^2 \beta_{\omega}^2}{M}\sigma^2 
+ 8\eta^2 L^2\psi(\sigma^2 + 3G^2).
\end{eqnarray*}

Furthermore, choosing $\eta = \min(\eta_0, \frac{1}{\sqrt{T}})$, we get the following rate:
\begin{eqnarray*}
&& \frac{1}{T}\sum_{t=0}^{T-1}\E{\|\nabla f(x_t)\|^2} \\
&\le& \max\left(1,\frac{1}{\eta_0\sqrt T}\right) \frac{4(f(x_0) - f^*)}{\sqrt T}
+ \frac{2L\sigma^2 }{M\sqrt{T}}
+ \frac{4 L^2 \beta_{\omega}^2 \sigma^2 }{MT}
+ \frac{8 L^2\psi(\sigma^2 + 3G^2)}{T} \\
&\le& \frac{4(f(x_0) - f^*)}{\sqrt T}
+ \frac{2L\sigma^2 }{M\sqrt{T}}
+ \frac{4(f(x_0) - f^*)}{\eta_0 T}
+ \frac{4 L^2 \beta_{\omega}^2 \sigma^2 }{MT}
+ \frac{8 L^2\psi(\sigma^2 + 3G^2)}{T} \\
&=& \frac{4}{\sqrt{T}}\left(f(x_0) - f^*
+ \frac{L\sigma^2 }{2M} \right)
+ \mathcal{O}\left(\frac{1+\beta_{\omega}^2 + \psi}{T}\right).
\end{eqnarray*}

\subsection{Key Lemmas}

\begin{lemma}\label{lem:z-x}
  For all $T\geq 1$, we have
  \begin{align}
    \sum_{t=0}^{T-1} \norm{z_t - x_t}^2
    \leq \frac{\eta^2\beta_{\omega}^2}{M} T\sigma^2 + \eta^2\beta_{\omega}^2 \sum_{t=0}^{T-1} \E\left\| \frac{1}{M}\sum_{m=1}^M \nabla f_m(x_{t}^m) \right\|^2,
  \end{align}
  where
  \begin{equation*}
    \beta_{\omega} \eqdef \sum_{j=1}^N \frac{\omega_j \beta_j}{1-\beta_j}. 
  \end{equation*}
\end{lemma}
\begin{proof}
  Since $u_{-1} = 0$, unrolling the update rule of momentum, for any $t\geq 0$ we get
  \begin{align*}
    u_t^j
    = \beta_j u^j_{t-1} + (1-\beta_j)g_t
    = (1-\beta_j)\sum_{\tau=0}^{t} \beta_j^{t-\tau} g^{\tau}.
  \end{align*}
Using this and the definition of the average iterates, we have
\begin{align*}
  z_t^j - x_t^j
  &= \frac{\beta_j}{1-\beta_j} (x_t^j - x_{t-1}^j) 
  = -\frac{\eta\beta_j}{1-\beta_j} u^j_t
  = -\eta\beta_j \sum_{\tau=0}^{t} \beta_j^{t-\tau} g_{\tau} \\
  z_t - x_t
  &= \E_j[z_t^j - x_t^j]
  = \E_j\left[ -\eta\beta_j \sum_{\tau=0}^{t} \beta_j^{t-\tau} g_{\tau} \right]
  = -\eta \sum_{\tau=0}^{t} \E_j\left[ \beta_j^{t-\tau+1} \right] g_{\tau} \\
  &= -\eta \sum_{\tau=0}^{t} \beta_{\omega}^{(t-\tau+1)} g_{\tau}, \quad \text{where } \beta_{\omega}^{(\tau)} = \E_j\left[ \beta_j^{\tau} \right] = \sum_{j=1}^N \omega_j\beta_j^{\tau}.
\end{align*}

Let us make another notation for the sum of weights in the above sum and bound it as follows:

\begin{eqnarray*}
s_t
&\eqdef& \sum_{\tau=0}^{t} \beta_{\omega}^{(t-\tau+1)}
= \sum_{\tau=0}^{t} \sum_{j=1}^N \omega_j\beta_j^{t-\tau+1} \\
&=& \sum_{j=1}^N \omega_j \sum_{\tau=0}^{t} \beta_j^{t-\tau+1}
= \sum_{j=1}^N \omega_j \frac{\beta_j-\beta_j^{t+2}}{1-\beta_j}
\le \sum_{j=1}^N \frac{\omega_j \beta_j}{1-\beta_j} \eqdef \beta_{\omega}.
\end{eqnarray*}

Using convexity of squared norm, we have
\begin{eqnarray*}
  \norm{z_t - x_t}^2
  &=& \eta^2 s_t^2 \left\| \sum_{\tau=0}^{t} \frac{\beta_{\omega}^{(t-\tau+1)}}{s_t} g_{\tau} \right\|^2
  \le \eta^2 s_t^2 \sum_{\tau=0}^{t} \frac{\beta_{\omega}^{(t-\tau+1)}}{s_t} \|g_{\tau}\|^2
  \le \eta^2 \beta_{\omega} \sum_{\tau=0}^{t} \beta_{\omega}^{(t-\tau+1)} \|g_{\tau}\|^2 \\
\end{eqnarray*}
Summing over the iterates yields
\begin{eqnarray*}
  \sum_{t=0}^{T-1} \E\norm{z_t - x_t}^2
  &\le& \eta^2\beta_{\omega} \sum_{t=0}^{T-1}\sum_{\tau=0}^{t} \beta_{\omega}^{(t-\tau+1)} \E\|g_{\tau}\|^2 \\
  &=& \eta^2\beta_{\omega} \sum_{\tau=0}^{T-1} \sum_{t=\tau}^{T-1} \beta_{\omega}^{(t-\tau+1)} \E\|g_{\tau}\|^2
  = \eta^2\beta_{\omega} \sum_{\tau=0}^{T-1} \left( \sum_{t=1}^{T-\tau} \beta_{\omega}^{(t)} \right) \E\|g_{\tau}\|^2 \\
  &=& \eta^2\beta_{\omega} \sum_{\tau=0}^{T-1} \left(\sum_{j=1}^N \omega_j\frac{\beta_j-\beta_j^{T-\tau+1}}{1-\beta_j}\right) \E\|g_{\tau}\|^2 \\
  &\le& \eta^2\beta_{\omega}^2 \sum_{\tau=0}^{T-1} \E\|g_{\tau}\|^2 \\
  &=& \eta^2\beta_{\omega}^2 \sum_{\tau=0}^{T-1} \E\left\| \frac{1}{M}\sum_{m=1}^M g_{\tau}^m - \nabla f_m(x_{\tau}^m) \right\|^2
  + \eta^2\beta_{\omega}^2 \sum_{\tau=0}^{T-1} \E\left\| \frac{1}{M}\sum_{m=1}^M \nabla f_m(x_{\tau}^m) \right\|^2 \\
  &=& \frac{\eta^2\beta_{\omega}^2}{M^2} \sum_{\tau=0}^{T-1} \sum_{m=1}^M \E\left\| g_{\tau}^m - \nabla f_m(x_{\tau}^m) \right\|^2
  + \eta^2\beta_{\omega}^2 \sum_{\tau=0}^{T-1} \E\left\| \frac{1}{M}\sum_{m=1}^M \nabla f_m(x_{\tau}^m) \right\|^2 \\
  &=& \frac{\eta^2\beta_{\omega}^2}{M} T\sigma^2
  + \eta^2\beta_{\omega}^2 \sum_{\tau=0}^{T-1} \E\left\| \frac{1}{M}\sum_{m=1}^M \nabla f_m(x_{\tau}^m) \right\|^2.
\end{eqnarray*}
\end{proof}

\begin{lemma} \label{lem:x-xm}
If $24\eta^2 L^2 \psi \le 1$, then
\begin{align*}
\frac{1}{MT} \sum_{t=0}^{T-1} \sum_{m=1}^{M} \E\norm{x_t - x_t^m}^2
\leq 12\eta^2 (B^2-1) \psi \cdot \frac{1}{T}\sum_{t=0}^{T-1}\E\|\nabla f(x_t)\|^2
+ 4\eta^2\psi(\sigma^2 + 3G^2),
\end{align*}
where
$$\psi = \frac{4(1-p_x)}{p_x^2} \cdot \sum_{j=1}^N \omega_j\frac{(1-\beta_j)(1-p_j)}{1-(1-p_j)\beta_j}$$
\end{lemma}
\begin{proof}
Let us expand the term $\E{\|x_{t+1} - x_{t+1}^m\|^2}$ using $x_{t+1}^m$'s probabilistic update rule:
\begin{eqnarray*}
\E{\|x_{t+1} - x_{t+1}^m\|^2}
&=& p_x\cdot 0 + (1-p_x)\cdot \E{\|x_{t} - \eta u_{t} - (x_{t}^m - \eta u_{t}^m) \|^2}\\
&=& (1-p_x)\cdot \E{\|x_{t} - x_{t}^m - \eta(u_t - u_t^m)\|^2}\\
&\le& (1-p_x)(1+s) \E{\|x_{t} - x_{t}^m\|^2} + \eta^2(1-p_x)(1+\nicefrac{1}{s})\E{\|u_t - u_t^m\|^2}\\
&\le& \eta^2 (1-p_x)(1+\nicefrac{1}{s}) \sum_{\tau=1}^t ((1-p_x)(1+s))^{t-\tau} \E{\|u_{\tau} - u_{\tau}^m\|^2}.
\end{eqnarray*}
where $s>0$ will be chosen later. Next we expand the term $\E{\|u_t^j - u_t^{j,m}\|^2}$ using $u_t^{j,m}$'s probabilistic update rule:
\begin{eqnarray*}
\E{\|u_t^j - u_t^{j,m}\|^2}
&=& p_j \cdot 0 + (1-p_j)\cdot \E{\left\|\frac{1}{M}\sum_{m=1}^M (\beta_j u_{t-1}^{j,m} + (1-\beta_j)g_{t-1}^m) - (\beta_j u_{t-1}^{j,m} + (1-\beta_j)g_{t-1}^m) \right\|^2} \\
&=& (1-p_j)\E{\left\|\beta_j (u_{t-1}^j - u_{t-1}^{j,m}) + (1-\beta_j)(g_{t-1} - g_{t-1}^m) \right\|^2} \\
&\le& (1-p_j)\beta_j\E{\|(u_{t-1}^j - u_{t-1}^{j,m}) \|^2}
+ (1-p_j)(1-\beta_j)\E{\| g_{t-1} - g_{t-1}^m \|^2} \\
&\le& (1-p_j)(1-\beta_j)\sum_{\tau=0}^{t-1} ((1-p_j)\beta_j)^{t-\tau-1} \E{\| g_{\tau} - g_{\tau}^m \|^2} \\
\E{\|u_t - u_t^m\|^2}
&\le& \sum_{j=1}^N \omega_j \E{\|u_t^j - u_t^{j,m}\|^2} \\
&\le& \sum_{j=1}^N \omega_j (1-p_j)(1-\beta_j) \sum_{\tau=0}^{t-1} ((1-p_j)\beta_j)^{t-\tau-1} \E{\| g_{\tau} - g_{\tau}^m \|^2} \\
&\le& \sum_{\tau=0}^{t-1} \left(\sum_{j=1}^N \omega_j (1-p_j)(1-\beta_j) ((1-p_j)\beta_j)^{t-\tau-1} \right) \E{\| g_{\tau} - g_{\tau}^m \|^2} \\
&\le& \sum_{\tau=0}^{t-1} \left(\sum_{j=1}^N \omega_j (1-p_j)(1-\beta_j) q_j^{t-\tau-1} \right) \E{\| g_{\tau} - g_{\tau}^m \|^2}.
\end{eqnarray*}

Denote $q_x = (1-p_x)(1+s),\;q'_x = (1-p_x)(1+\nicefrac{1}{s})$ and $q_j = (1-p_j)\beta_j$. Combining the previous two bounds, we get
\begin{eqnarray}
&& \frac{1}{M}\sum_{m=1}^M\E{\|x_t - x_t^m\|^2} \nonumber \\
&\le& \eta^2 q'_x \sum_{\tau=1}^t q_x^{t-\tau} \frac{1}{M}\sum_{m=1}^M\E{\|u_{\tau} - u_{\tau}^m\|^2} \label{double-geom-sum} \\
&\le& \eta^2 q'_x \sum_{\tau=1}^t q_x^{t-\tau} \frac{1}{M}\sum_{m=1}^M \sum_{\nu=0}^{\tau-1} \left(\sum_{j=1}^N \omega_j (1-p_j)(1-\beta_j) q_j^{\tau-\nu-1} \right) \E{\| g_{\nu} - g_{\nu}^m \|^2} \nonumber \\
&=& \eta^2 q'_x \sum_{j=1}^N \omega_j (1-p_j)(1-\beta_j) \sum_{\tau=1}^t \sum_{\nu=0}^{\tau-1} q_x^{t-\tau-1} q_j^{\tau-\nu} \left[\frac{1}{M}\sum_{m=1}^M\E{\| g_{\nu} - g_{\nu}^m \|^2}\right] \nonumber \\
&=& \eta^2 q'_x \sum_{j=1}^N \omega_j (1-p_j)(1-\beta_j) \sum_{\nu=0}^{t-1} \sum_{\tau=\nu+1}^{t} q_x^{t-\tau} q_j^{\tau-\nu-1} \left[\frac{1}{M}\sum_{m=1}^M\E{\| g_{\nu} - g_{\nu}^m \|^2}\right] \nonumber \\
&=& \eta^2 q'_x \sum_{j=1}^N \omega_j (1-p_j)(1-\beta_j) \sum_{\nu=0}^{t-1}
\frac{q_x^{t-\nu} - q_j^{t-\nu}}{q_x-q_j} \left[\frac{1}{M}\sum_{m=1}^M\E{\| g_{\nu} - g_{\nu}^m \|^2}\right], \nonumber \\
&=& \eta^2 q'_x \sum_{j=1}^N \omega_j (1-\beta_j)(1-p_j) \sum_{\nu=0}^{t-1}
\frac{q_x^{t-\nu} - q_j^{t-\nu}}{q_x-q_j} \left[\frac{1}{M}\sum_{m=1}^M\E{\| g_{\nu} - g_{\nu}^m \|^2}\right], \nonumber.
\end{eqnarray}

Next, we bound the gradient term above.
\begin{eqnarray*}
\frac{1}{M}\sum_{m=1}^M \E{\| g_t^m - g_t \|^2}
&=& \frac{1}{M}\sum_{m=1}^M \E{\left\| g_t^m - \frac{1}{M}\sum_{i=1}^K g_t^{i} \right\|^2} \\
&\le& \frac{2}{K}\sum_{m=1}^M \E{\left\| g_t^m - \nabla f_m(x_t^m) - \frac{1}{M}\sum_{i=1}^M (g_{t}^{i} - \nabla f_{i}(x_t^{i})) \right\|^2} \\
&&\quad +\; \frac{2}{M}\sum_{m=1}^M\E{\left\|\nabla f_m(x_t^m) - \frac{1}{M}\sum_{i=1}^M \nabla f_{i}(x_t^{i}) \right\|^2} \\
\textrm{(Lemma \ref{lem:het-var})}&\le& \frac{2}{M}\sum_{m=1}^M \E{\left\| g_t^m - \nabla f_m(x_t^m) \right\|^2}
- 2 \E{\left\| \frac{1}{M}\sum_{m=1}^M (g_{t}^m - \nabla f_m(x_t^m)) \right\|^2} \\
&&\quad +\; \frac{12L^2}{M}\sum_{m=1}^{M} \E\norm{x_t - x_t^m}^2 + 6(B^2-1) \E\|\nabla f(x_t)\|^2 + 6G^2 \\
&\le& 2\sigma^2 + \frac{12L^2}{M}\sum_{m=1}^{M} \E\norm{x_t - x_t^m}^2 + 6(B^2-1) \E\|\nabla f(x_t)\|^2 + 6 G^2.
\end{eqnarray*}

Averaging over the iterates and plugging this bound to the previous one, we get
\begin{eqnarray*}
&& \frac{1}{MT}\sum_{t=0}^{T-1}\sum_{m=1}^M\E{\|x_t - x_t^m\|^2} \\
&\le& \frac{1}{MT}\sum_{t=1}^T\sum_{m=1}^M\E{\|x_t - x_t^m\|^2} \\
&\le& \frac{\eta^2 q'_x}{T} \sum_{j=1}^N \omega_j (1-\beta_j)(1-p_j) \sum_{t=1}^T\sum_{\tau=0}^{t-1}
\frac{q_x^{t-\tau} - q_j^{t-\tau}}{q_x-q_j} \left[\frac{1}{M}\sum_{m=1}^M\E{\| g_{\tau} - g_{\tau}^m \|^2}\right] \\
&=& \frac{\eta^2 q'_x}{T} \sum_{j=1}^N \omega_j (1-\beta_j)(1-p_j) \sum_{\tau=0}^{T-1}\sum_{t=\tau+1}^{T}
\frac{q_x^{t-\tau} - q_j^{t-\tau}}{q_x-q_j} \left[\frac{1}{M}\sum_{m=1}^M\E{\| g_{\tau} - g_{\tau}^m \|^2}\right] \\
&=& \frac{\eta^2 q'_x}{T} \sum_{j=1}^N \frac{\omega_j (1-\beta_j)(1-p_j)
}{q_x-q_j} \sum_{\tau=0}^{T-1} \left( \frac{q_x(1-q_x^{T-\tau})}{1-q_x} - \frac{q_j(1-q_j^{T-\tau})}{1-q_j} \right) \left[\frac{1}{M}\sum_{m=1}^M\E{\| g_{\tau} - g_{\tau}^m \|^2}\right] \\
&\le& \frac{\eta^2 q'_x}{T} \sum_{j=1}^N \frac{\omega_j (1-\beta_j)(1-p_j)
}{q_x-q_j} \sum_{\tau=0}^{T-1} \left( \frac{q_x}{1-q_x} - \frac{q_j}{1-q_j} \right) \left[\frac{1}{M}\sum_{m=1}^M\E{\| g_{\tau} - g_{\tau}^m \|^2}\right] \\
&=& \frac{\eta^2 q'_x}{T} \sum_{j=1}^N \frac{\omega_j (1-\beta_j)(1-p_j)
}{(1-q_x)(1-q_j)} \sum_{\tau=0}^{T-1} \left[\frac{1}{M}\sum_{m=1}^M\E{\| g_{\tau} - g_{\tau}^m \|^2}\right] \\
\end{eqnarray*}

Now, let us optimize the factor
$$
\frac{q'_x}{1-q_x}
= \frac{(1-p_x)(1+\nicefrac{1}{s})}{1-(1-p_x)(1+s)}
$$
by choosing optimal value for $s$ introduced earlier. By the first order optimality condition, we find that the optimal value is $s^* = \frac{1}{\sqrt{1-p_x}}-1$. Hence, the minimal value of the factor is
\begin{eqnarray*}
\frac{q'_x}{1-q_x}
&=& \frac{1-p_x}{(1-\sqrt{1-p_x})^2} \\
&=& \frac{(1-p_x)(1-\sqrt{1-p_x})^2}{(1-\sqrt{1-p_x})^2(1+\sqrt{1-p_x})^2}
= \frac{(1-p_x)(1+\sqrt{1-p_x})^2}{p_x^2}
\le \frac{4(1-p_x)}{p_x^2}.
\end{eqnarray*}

Letting
$$
\psi = \frac{4(1-p_x)}{p_x^2} \sum_{j=1}^N \omega_j\frac{(1-\beta_j)(1-p_j)}{1-q_j}
= \frac{4(1-p_x)}{p_x^2} \sum_{j=1}^N \omega_j\frac{(1-\beta_j)(1-p_j)}{1-(1-p_j)\beta_j}
$$
and continuing the chain of bounds, we get
\begin{eqnarray*}
&& \frac{1}{MT}\sum_{t=0}^{T-1}\sum_{m=1}^M\E{\|x_t - x_t^m\|^2} \\
&\le& \eta^2\psi \cdot \frac{1}{T}\sum_{t=0}^{T-1} \left[\frac{1}{K}\sum_{m=1}^M \E{\| g_t - g_t^m \|^2}\right] \\
&\le& \eta^2\psi \cdot \frac{1}{T}\sum_{t=0}^{T-1} \left[ \frac{12L^2}{M}\sum_{m=1}^{M} \E\norm{x_t - x_{t}^m}^2 + 6(B^2-1) \E\|\nabla f(x_t)\|^2 + 2\sigma^2 + 6 G^2 \right] \\
&\le& 12\eta^2 L^2 \psi \cdot \frac{1}{TM}\sum_{t=0}^{T-1} \sum_{m=1}^{M} \E\norm{x_t - x_{t}^m}^2 \\
&& +\; 6\eta^2 (B^2-1) \psi \cdot \frac{1}{T}\sum_{t=0}^{T-1}\E\|\nabla f(x_t)\|^2
+ 2\eta^2\psi(\sigma^2 + 3G^2).
\end{eqnarray*}
Assuming $12\eta^2 L^2 \psi \le \nicefrac{1}{2}$ and reordering the first term in the bound, we arrive
$$
\frac{1}{MT}\sum_{t=0}^{T-1}\sum_{m=1}^M\E{\|x_t - x_t^m\|^2}
\le 12\eta^2 (B^2-1) \psi \cdot \frac{1}{T}\sum_{t=0}^{T-1}\E\|\nabla f(x_t)\|^2
+ 4\eta^2\psi(\sigma^2 + 3G^2).
$$
\end{proof}

\begin{lemma}\label{lem:het-var} Under smoothness and bounded heterogeneity assumptions \ref{ass:smooth} and \ref{ass:het}, we have
\begin{equation*}
\frac{1}{M}\sum_{m=1}^M \left\|\nabla f_m(y^m) - \frac{1}{K}\sum_{i=1}^K \nabla f_i(y^i)\right\|^2
\leq \frac{6L^2}{M}\sum_{m=1}^{M}\norm{y - y^m}^2 + 3(B^2-1)\|\nabla f(y)\|^2 + 3 G^2,
\end{equation*} 
for any $y^1,\dots,y^m\in\R^d$ and $y = \E_m[y^m]$.
\end{lemma}
\begin{proof}
    The proof follows from Lemma 5 of \citep{DES-LOC} as the result does not depend on the optimizer.
\end{proof}
\section{Wall-Clock Time Modeling}\label{app:sec:wall_time_model}

To assess the practical benefits of our proposal, we analyze its impact on total wall-clock time by modeling two distinct synchronization strategies: a simple unified frequency approach and a desynchronized approach based on optimizer state half-lives. We adopt the model from \texttt{DES-LOC}~\citep{DES-LOC} for estimating total training time.

\subsection{Wall-Clock Time Model}
The total wall-clock time is modeled as the sum of computational and communication time: $t_{\text{total}} = t_{\text{compute}} + t_{\text{comms}}$. The computation time, $t_{\text{compute}}$, is a function of model and dataset size, while the communication time, $t_{\text{comms}}$, depends on the number and size of synchronization events.

For a training process of $T$ total steps, the communication time for an AllReduce operation \citep{Horovod} depends on the payload size, number of workers $M$, bandwidth $B$, and latency $l$. The total time for different methods and strategies is:
\begin{itemize}
    \item[] \textbf{Unified Frequency Methods:} Parameters and all optimizer states are synchronized together every $K$ steps. The total payload is $3d$ (for parameters, first and second momenta). This applies to \localadam and a baseline version of our method, \method (Unified).
    \begin{equation}
    t_{\text{total}, \text{Unified}} = t_{\text{compute}} + \frac{T}{K} \cdot \bigg[ \frac{2(3d)}{B} \bigg(1 - \frac{1}{M}\bigg) + l \bigg]
    \end{equation}
    \item[] \textbf{Half-Life Based Methods:} Parameters ($K_x$), first momentum ($K_u$), and second momentum ($K_v$) are synchronized at different frequencies. This applies to \desloc and our proposed method, \method (Half-Life).
    \begin{equation} 
    t_{\text{total}, \text{Half-Life}} = t_{\text{compute}} + \bigg(\frac{T}{K_x} + \frac{T}{K_u} + \frac{T}{K_v}\bigg) \cdot \bigg[ \frac{2d}{B} \bigg(1 - \frac{1}{M}\bigg) + l \bigg]
    \end{equation}
\end{itemize}
\textbf{Limitation:} This model does not account for any potential overlap between computation and communication.

\subsection{Experimental Configuration}

We compare the two synchronization strategies. The \textbf{Unified Frequency} strategy serves as a baseline, where all states are synchronized together every $K_x=32$ steps. This includes \localadam, \localadopt, and a variant of our method, \method (Unified), which uses a high $\beta_1$ value but is forced to sync at the same frequent rate as its parameters. These methods have equivalent communication costs and will overlap for an iso-token budget, however, the results in \cref{sec:evaluation} show that \method achieves the same perplexity as \localadopt in many fewer optimization steps, outperforming on time-to-perplexity metrics.

The \textbf{Half-Life Based} strategy aims to improve efficiency by synchronizing states less frequently if they change slowly. The synchronization frequency is set based on the state's half-life, $\tau_{0.5}(\beta) = \ln(0.5)/\ln(\beta)$. This includes \desloc and \method (Half-Life). The quasi-hyperbolic (QH) configuration of \method allows it to use an extremely high $\beta_1=0.999$, leading to a very long half-life and thus a much lower communication frequency for its first momentum. We use $\beta_2=0.999$ for ADAM variants and $\beta_2=0.9999$ for ADOPT variants.

Table \ref{tab:hyperparam_configs} details the configurations for both strategies.

\begin{table}[h!]
\centering
\caption{Hyperparameter configurations and synchronization frequencies ($K$) for modeled methods, grouped by synchronization strategy. For the Half-Life strategy, momentum frequencies are set to the closest power of two to their half-life.}
\label{tab:hyperparam_configs}
\resizebox{\textwidth}{!}{%
\begin{tabular}{@{}llcccccc@{}}
\toprule
\textbf{Strategy} & \textbf{Method} & \textbf{$\omega$ Values} & \textbf{$\beta_1$ Values} & \textbf{$\beta_2$ Value} & \textbf{Sync Freq. $K_{u_1}$} & \textbf{Sync Freq. $K_v$} \\ \midrule
\multicolumn{7}{l}{\textit{Unified Frequency (All states sync every $K_x=32$ steps)}} \\ \midrule
Unified & \localadam             & N/A                      & \{0.95\}                 & 0.99                 & 32                         & 32                       \\
Unified & \localadopt            & N/A                      & \{0.95\}                 & 0.9999               & 32                         & 32                       \\
Unified & \methodadam (Unified)  & \{0.95\}                 & \{0.999\}                & 0.999                & 32                         & 32                       \\
Unified & \methodadopt (Unified) & \{0.95\}                 & \{0.999\}                & 0.9999               & 32                         & 32                       \\ \midrule
\multicolumn{7}{l}{\textit{Half-Life Based Frequency (States sync at different rates from $K_x=32$)}} \\ \midrule
Half-Life & \deslocadam            & N/A                      & \{0.95\}                 & 0.99                 & 32                         & 69                       \\
Half-Life & \deslocadopt           & N/A                      & \{0.95\}                 & 0.9999               & 32                         & 6931                     \\
Half-Life & \methodadam (Half-Life)  & \{0.95\}                 & \{0.999\}                & 0.999                & 693                        & 693                      \\
Half-Life & \methodadopt (Half-Life) & \{0.95\}                 & \{0.999\}                & 0.9999               & 693                        & 6931                     \\ \bottomrule
\end{tabular}%
}
\end{table}

\subsection{Modeling Results}
The following figures present the estimated wall-clock time and communication costs when training a 1B model on 4 H100 machines with a batch size of $2$M tokens and sequence length of $2048$. The results demonstrate that \method significantly reduces communication cost with both strategies, with the half-life one being generally more effective.

\begin{figure}[H]
    \centering
    \subfloat[\adam Variants]{\includegraphics[width=0.49\columnwidth]{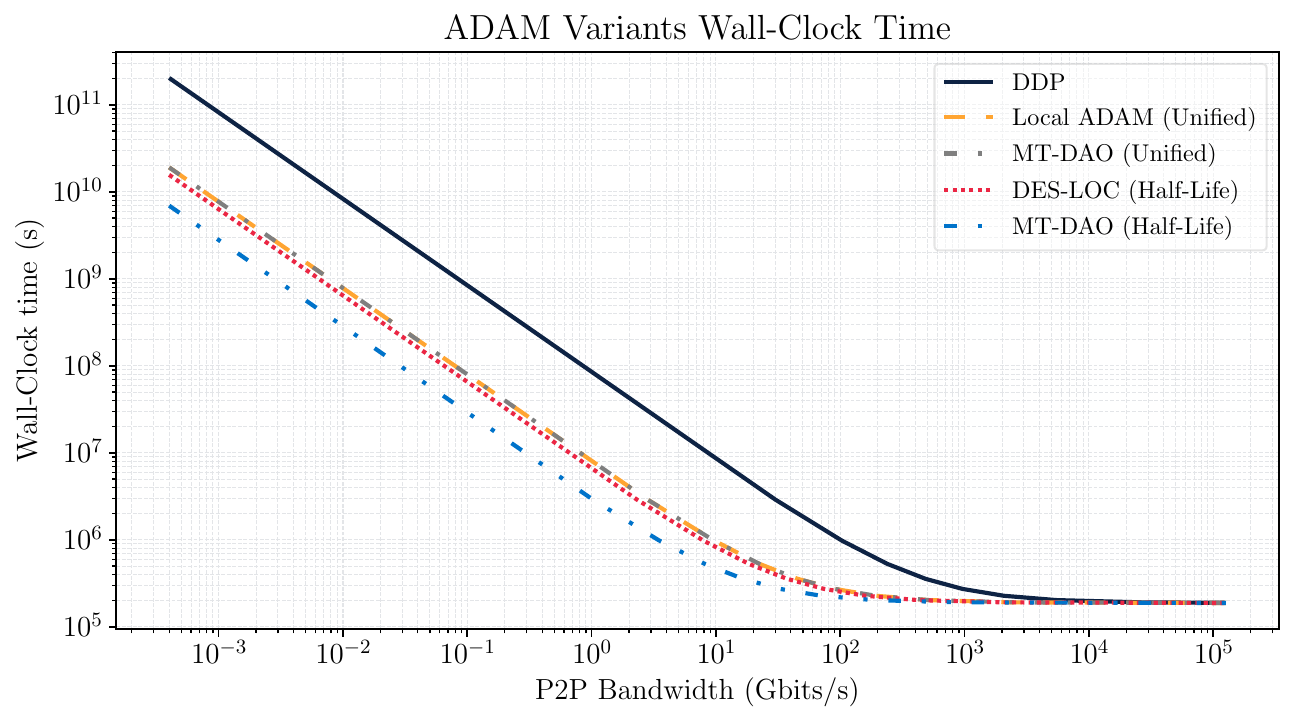}} \hfill
    \subfloat[\adopt Variants]{\includegraphics[width=0.49\columnwidth]{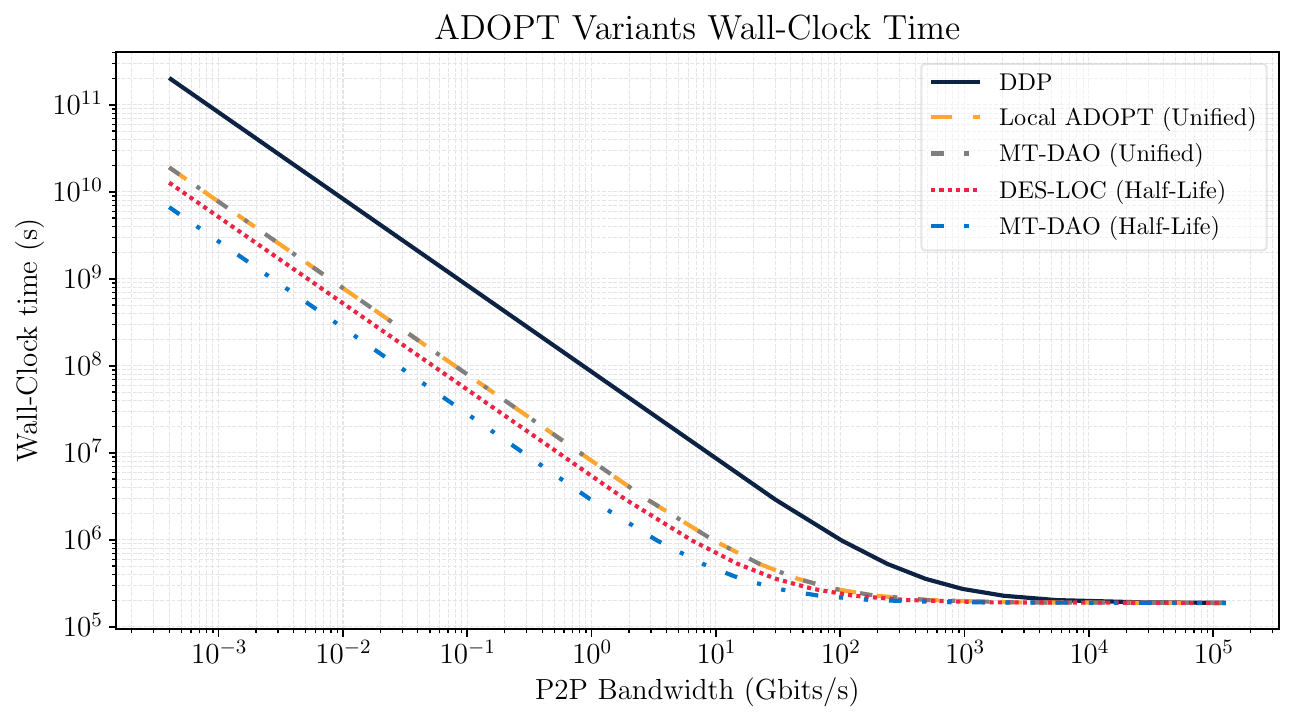}}
    \caption{Estimated total wall-clock time as a function of interconnect bandwidth. For both (a) \adam and (b) \adopt, methods using the Half-Life strategy outperform those using a Unified frequency.}
    \label{fig:wall_clock_time}
\end{figure}

\begin{figure}[H]
    \centering
    \subfloat[\adam Variants]{\includegraphics[width=0.49\columnwidth]{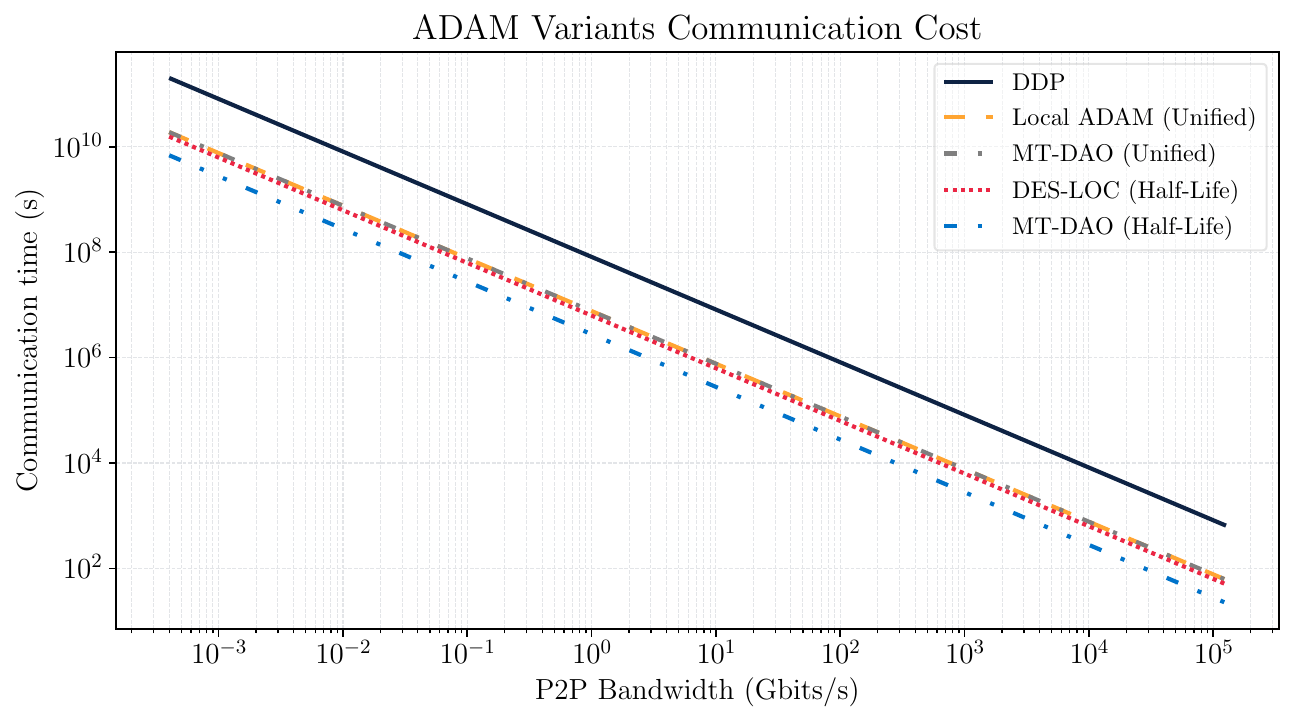}} \hfill
    \subfloat[\adopt Variants]{\includegraphics[width=0.49\columnwidth]{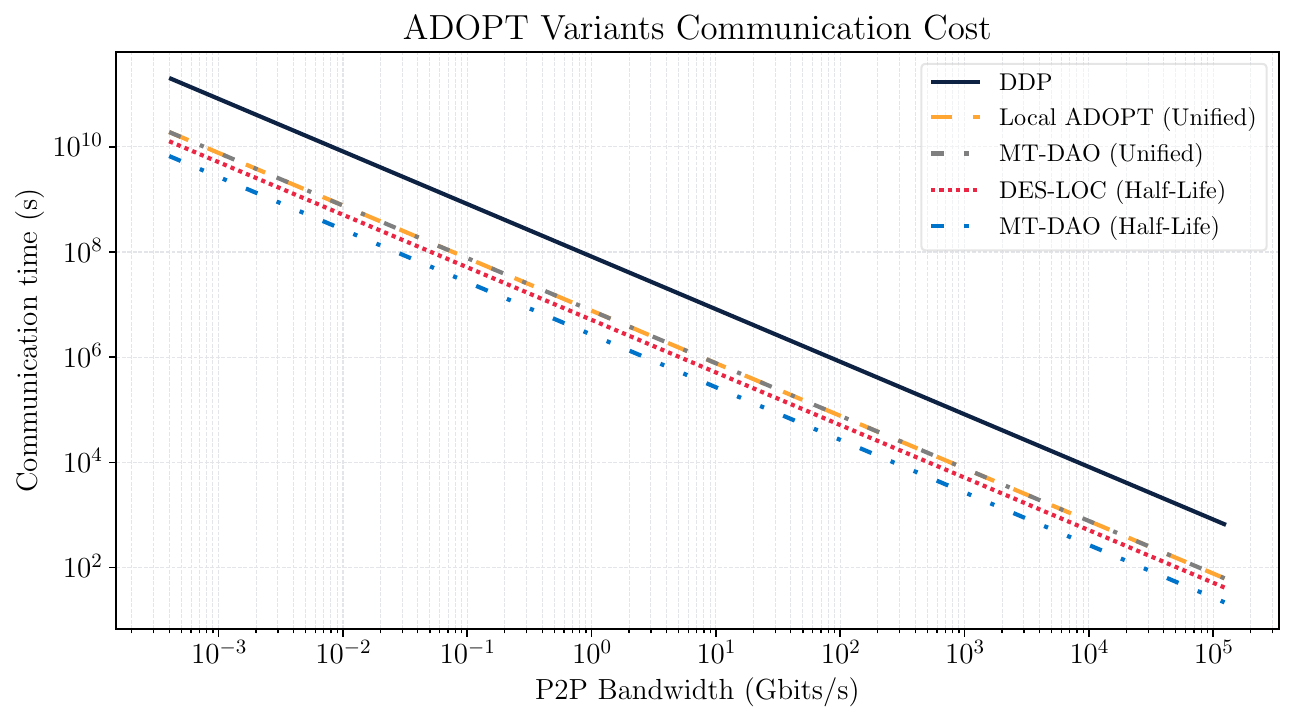}}
    \caption{Estimated total communication time. The plots clearly distinguish the two strategies. The Unified frequency methods (\localadam and \method (Unified)) have identical high costs. The Half-Life methods are more efficient, with \method (Half-Life) being the most efficient due to its ability to leverage a high-$\beta$ momentum that requires infrequent updates.}
    \label{fig:comms_cost}
\end{figure}

\takeawaybox[Takeaway:]{%
\method can significantly reduce communication costs across bandwidths.
}
\section{Derivations of Mutual Information and Variance}\label{app:sec:small_derivations}

This section provides the detailed derivations for the expressions referenced in the main text.

\subsection{Variance of Local Momentum}

The variance of the final local momentum, $\text{Var}(u_{t+K})$, is derived under the assumption that the stochastic gradients $g_t$ are independent and identically distributed random variables with variance $\sigma_m^2$.

The unrolled momentum update over $K$ local steps is given by:
$$u_{t+K} = \beta^K u_t + (1-\beta)\sum_{k=0}^{K-1} \beta^k g_{t+K-k}$$
The variance is calculated with respect to the randomness in the local gradients $\{g_{t+1}, \dots, g_{t+K}\}$. The initial momentum $u_t$ is treated as a constant, as it is a synchronized state before local updates begin.

Applying the variance operator:
$$\text{Var}(u_{t+K}) = \text{Var}\left(\beta^K u_t + (1-\beta)\sum_{k=0}^{K-1} \beta^k g_{t+K-k}\right)$$
Since $u_t$ is constant, $\text{Var}(\beta^K u_t) = 0$. Using the property $\text{Var}(aX) = a^2\text{Var}(X)$:
$$\text{Var}(u_{t+K}) = (1-\beta)^2 \text{Var}\left(\sum_{k=0}^{K-1} \beta^k g_{t+K-k}\right)$$
Given the assumption that the gradients $g_{t+i}$ are independent, the variance of their weighted sum is the weighted sum of their variances, where weights are squared:
$$\text{Var}\left(\sum_{k=0}^{K-1} \beta^k g_{t+K-k}\right) = \sum_{k=0}^{K-1} \text{Var}(\beta^k g_{t+K-k}) = \sum_{k=0}^{K-1} (\beta^k)^2 \text{Var}(g_{t+K-k})$$
Assuming each local gradient has variance $\sigma_m^2$:
$$\text{Var}\left(\sum_{k=0}^{K-1} \beta^k g_{t+K-k}\right) = \sum_{k=0}^{K-1} \beta^{2k} \sigma_m^2 = \sigma_m^2 \sum_{k=0}^{K-1} (\beta^2)^k$$
The summation is a finite geometric series, $\sum_{i=0}^{n-1} r^i = \frac{1-r^n}{1-r}$. With $r = \beta^2$ and $n=K$:
$$\sum_{k=0}^{K-1} (\beta^2)^k = \frac{1 - (\beta^2)^K}{1-\beta^2} = \frac{1 - \beta^{2K}}{1-\beta^2}$$
Substituting this back into the expression for $\text{Var}(u_{t+K})$:
$$\text{Var}(u_{t+K}) = (1-\beta)^2 \sigma_m^2 \frac{1 - \beta^{2K}}{1-\beta^2}$$
By factoring the denominator $1-\beta^2 = (1-\beta)(1+\beta)$, we can simplify the expression:
$$\text{Var}(u_{t+K}) = (1-\beta)^2 \sigma_m^2 \frac{1 - \beta^{2K}}{(1-\beta)(1+\beta)} = \frac{1-\beta}{1+\beta}(1-\beta^{2K})\sigma_m^2$$
This completes the derivation.

\subsection{Mutual Information}

The mutual information $I(U_{t+K}; U_t)$ is derived by modeling the momentum states as multivariate Gaussian random vectors. The model for the update process is:
$$U_{t+K} = \beta^K U_t + L$$
The following assumptions are made:
\begin{enumerate}
    \item The initial momentum $U_t$ is a Gaussian random vector with zero mean and covariance $\Sigma_{U_t}$, i.e., $U_t \sim \mathcal{N}(0, \Sigma_{U_t})$.
    \item The accumulated local gradient noise $L$ is a Gaussian random vector with zero mean and covariance $\Sigma_L$, i.e., $L \sim \mathcal{N}(0, \Sigma_L)$.
    \item $U_t$ and $L$ are statistically independent.
\end{enumerate}
The mutual information between two random vectors $X$ and $Y$ is defined as $I(X; Y) = h(Y) - h(Y|X)$, where $h(\cdot)$ is the differential entropy. For a $d$-dimensional Gaussian vector $Z \sim \mathcal{N}(\mu, \Sigma)$, the entropy is $h(Z) = \frac{1}{2} \log \det(2\pi e \Sigma)$.

First, we determine the distribution of $U_{t+K}$. As a linear combination of independent Gaussian vectors, it is also Gaussian.
\begin{itemize}
    \item \textbf{Mean}: $\mathbb{E}[U_{t+K}] = \mathbb{E}[\beta^K U_t + L] = \beta^K \mathbb{E}[U_t] + \mathbb{E}[L] = 0$.
    \item \textbf{Covariance}: $\text{Cov}(U_{t+K}) = \text{Cov}(\beta^K U_t + L)$. Due to the independence of $U_t$ and $L$:
    $$
    \Sigma_{U_{t+K}} = \text{Cov}(\beta^K U_t) + \text{Cov}(L) = \beta^{2K}\Sigma_{U_t} + \Sigma_L
    $$
\end{itemize}
Thus, $U_{t+K} \sim \mathcal{N}(0, \beta^{2K}\Sigma_{U_t} + \Sigma_L)$.

The entropy of $U_{t+K}$ is:
$$h(U_{t+K}) = \frac{1}{2} \log \det\left(2\pi e (\beta^{2K}\Sigma_{U_t} + \Sigma_L)\right)$$
Next, we determine the conditional entropy $h(U_{t+K}|U_t)$. The distribution of $U_{t+K}$ conditioned on a specific value $U_t = u_t$ is:
$$U_{t+K} | U_t=u_t \sim \mathcal{N}(\beta^K u_t, \Sigma_L)$$
The entropy of this conditional distribution is:
$$h(U_{t+K} | U_t=u_t) = \frac{1}{2} \log \det(2\pi e \Sigma_L)$$
Since this expression does not depend on the specific value $u_t$, the conditional entropy $h(U_{t+K}|U_t)$ is the same.

Now, we compute the mutual information:
$$I(U_{t+K}; U_t) = h(U_{t+K}) - h(U_{t+K}|U_t)$$
$$I(U_{t+K}; U_t) = \frac{1}{2} \log \det\left(2\pi e (\beta^{2K}\Sigma_{U_t} + \Sigma_L)\right) - \frac{1}{2} \log \det(2\pi e \Sigma_L)$$
Using the logarithmic property $\log a - \log b = \log(a/b)$:
$$I(U_{t+K}; U_t) = \frac{1}{2} \log \left( \frac{\det(2\pi e (\beta^{2K}\Sigma_{U_t} + \Sigma_L))}{\det(2\pi e \Sigma_L)} \right)$$
The constant factors $(2\pi e)^d$ cancel out. Using the determinant property $\frac{\det(A)}{\det(B)} = \det(AB^{-1})$:
$$I(U_{t+K}; U_t) = \frac{1}{2} \log \det\left( (\beta^{2K}\Sigma_{U_t} + \Sigma_L) \Sigma_L^{-1} \right)$$
Distributing $\Sigma_L^{-1}$ inside the determinant:
$$I(U_{t+K}; U_t) = \frac{1}{2} \log \det\left( \beta^{2K}\Sigma_{U_t}\Sigma_L^{-1} + \Sigma_L\Sigma_L^{-1} \right)$$
$$I(U_{t+K}; U_t) = \frac{1}{2} \log \det\left( I + \beta^{2K}\Sigma_{U_t}\Sigma_L^{-1} \right)$$
This completes the derivation.

\section{Extended Related Work}\label{app:sec:extended_related_work}

\paragraph{Strategies for Communication-Efficient Distributed Training.}
A substantial body of research aims to curtail communication overhead in distributed training, primarily by either reducing the frequency of synchronizations or compressing the data transmitted per round. The first approach, often termed periodic or local SGD, involves performing multiple local optimization steps between global aggregations. This strategy has been extensively analyzed in both IID and non-IID contexts (see \citet{kairouz2021advancesopenproblemsfederated} for a survey and \citet{lin2018don}). In the realm of foundation-model pre-training, methods like \textbf{DiLoCo} \citep{DiLoCoScalingLaws} have shown that infrequent synchronization can, with careful tuning, achieve performance comparable to or better than standard data parallelism, with scaling laws characterizing its behavior across model sizes \citep{DiLoCoScalingLaws}. This paradigm has also been adapted for federated-style pre-training \citep{Photon} and variants with overlapping or eager updates \citep{StreamingDiLoCo,kale2025eager}. The second strategy involves compressing communication payloads. Techniques range from randomized quantization (\textbf{QSGD}) \citep{QSGD_Alistarh_2017} and sparse updates tailored for non-IID data (\textbf{STC}, \textbf{ZeroFL}) \citep{sattler2019robust,zerofl} to one-bit aggregation (\textbf{signSGD-MV}) \citep{signSGD}. In practice, these two strategies are often combined; for instance, \textbf{FedPAQ} integrates local training with quantization and partial participation to provide strong theoretical guarantees \citep{reisizadeh2020fedpaq}.

\paragraph{Multi-Timescale Momentum for Temporal Mismatches.}
The temporal discrepancy between frequent local updates and infrequent global synchronizations creates a need for optimizers that can integrate information across different timescales. Standard momentum, while beneficial in low-curvature landscapes \citep{NesterovIlya}, imposes a compromise: low decay values are responsive but slow, whereas high decay values are fast but prone to oscillations \citep{AggMo}. A single exponential moving average (EMA) cannot effectively weight both recent and distant gradients \citep{AdEMAMix}. Multi-timescale optimizers address this limitation. \textbf{Quasi-Hyperbolic Momentum (QHM)} decouples the current gradient's weight from the momentum decay rate ($\beta$) \citep{QHM}, recovering methods like Nesterov and Triple Momentum \citep{Scoy2018TripleMomentum}. \textbf{Aggregated Momentum (AggMo)} maintains and averages multiple momentum buffers with distinct $\beta$ values, using faster-decaying terms to passively damp oscillations caused by slower, more aggressive terms \citep{AggMo}. Similarly, \textbf{AdEMAMix} mixes a fast EMA with an ultra-slow one (e.g., $\beta_3 = 0.9999$), demonstrating that long-term gradient memory significantly reduces catastrophic forgetting in language models \citep{AdEMAMix}. This principle of leveraging multiple timescales is also present in other contexts. Optimizers like \texttt{Grokfast} \citep{Lee2024Grokfast} and \texttt{AdMeta} \citep{Chen2023bAdMeta} employ nested EMAs for different purposes, providing orthogonal evidence for the value of long-term momentum. While these methods have shown promise in step-wise synchronous training, their potential to resolve the temporal mismatch in communication-efficient distributed optimization remains largely unexplored.

\paragraph{Perspectives from Federated Optimization.}
The field of Federated Learning (FL), particularly in the cross-device setting, offers a rich history of methods for managing statistical heterogeneity and communication constraints, which are central challenges. The foundational \textbf{FedAvg} algorithm \citep{fedavg} has inspired numerous successors (see survey by \citet{kairouz2021advancesopenproblemsfederated}). To counteract \emph{client drift} caused by non-IID data, \textbf{FedProx} introduces a proximal regularizer for stability \citep{FedProx}, \textbf{SCAFFOLD} employs control variates to reduce gradient variance \citep{pmlr-v119-karimireddy20a}, and \textbf{FedNova} normalizes local updates to correct for objective inconsistency \citep{Wang2020TacklingOptimization}. Server-side momentum (\textbf{FedAvgM}) has also been shown to stabilize aggregation under data skew \citep{fedavgm}. Adaptive methods have been extended to this setting in \textbf{Adaptive Federated Optimization} (\textsc{FedOpt}), which provides nonconvex guarantees for \emph{FedAdam}, \emph{FedYogi}, and \emph{FedAdagrad} \citep{Adam}. Furthermore, \textbf{Mime} adapts centralized algorithms to FL by marrying control variates with server statistics \citep{mime}. Personalization techniques, such as meta-learning-based \textbf{Per-FedAvg} \citep{fallah2020personalized} and \textbf{FedL2P} \citep{lee2023fedl2p} or the regularized \textbf{Ditto} \citep{li2021ditto}, complement these global models by improving per-client utility.

\paragraph{Orthogonal Approaches in Payload Compression and Optimizer Design.}
Orthogonal to reducing synchronization frequency, another line of work focuses on compressing the communication payload itself, often in combination with periodic training. Foundational methods include quantization, as in \textbf{QSGD} \citep{QSGD_Alistarh_2017}, and sparsification, as in \textbf{Deep Gradient Compression} \citep{DeepGradientCompression_Lin_2017}, with convergence analyses providing theoretical grounding \citep{alistarh2018convergence}. More recent work like \textbf{CocktailSGD} combines random and top-$k$ sparsification with quantization for aggressive compression during LLM fine-tuning \citep{wang2023cocktailsgd}. Beyond compressing gradients, some methods compress the optimizer \emph{states}. For instance, \textbf{LDAdam} performs adaptive updates using low-rank approximations of gradient statistics \citep{LDAdam_Robert_2023}, while \textbf{DeMo} decouples momentum across workers and communicates only selected components \citep{peng2024decoupled}. Other advanced optimizers aim for stability and efficiency through different mechanisms; for example, \texttt{Lion} uses a sign function with interpolated momentum \citep{Chen2023Lion}, and \texttt{Sophia} employs a Hessian-based pre-conditioner to temper step sizes in high-curvature directions \citep{Liu2023Sophia}. These approaches are generally compatible with and can be composed with infrequent synchronization strategies.

\section{Additional Results}\label{app:additional_results}
To investigate the stability of \method under varied momentum parameterizations, we now examine its performance in a fast $\beta$ regime. Figure \ref{fig:toy_example_method_fast} presents the results of this comparison, plotting both the convergence rate in terms of distance to the optimum and the optimization trajectories on the function's contour plot.

\begin{figure}[h] \centering \begin{subfigure}[b]{0.48\textwidth} \centering \includegraphics[width=\linewidth]{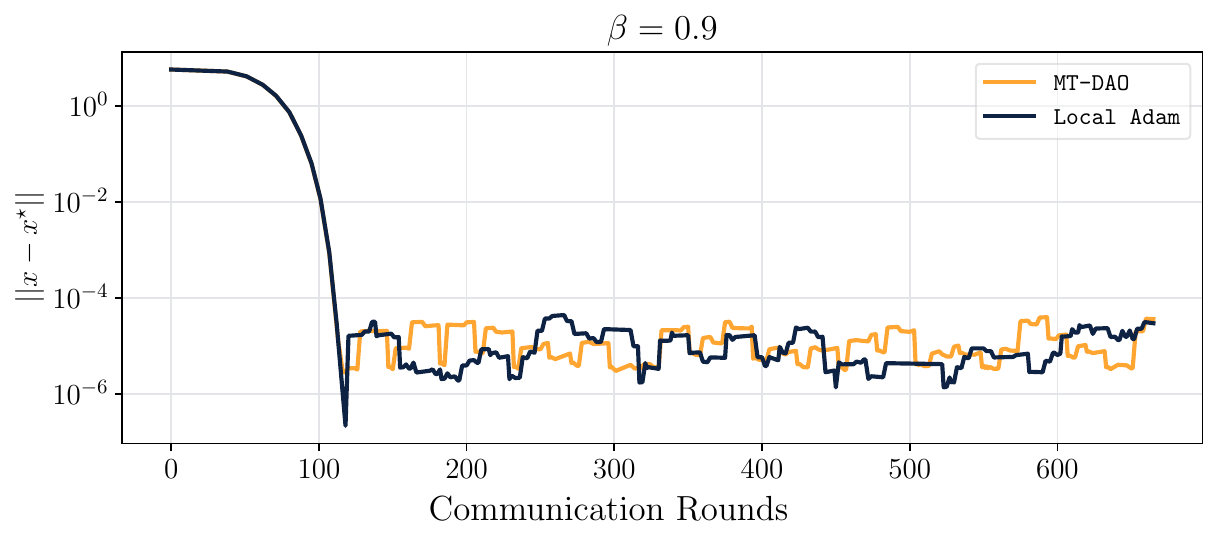} \caption{Distance to optimum vs. steps.} \label{fig:toy_dist_vs_steps_fast} \end{subfigure} \hfill %
\begin{subfigure}[b]{0.48\textwidth} \centering \includegraphics[width=\linewidth]{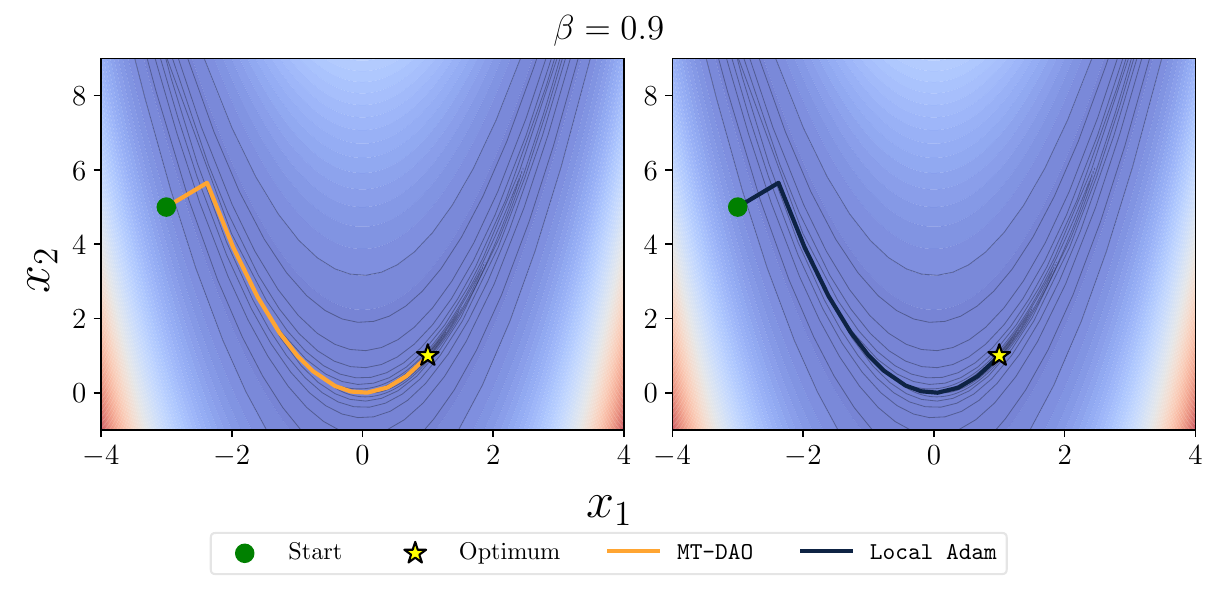} \caption{Contour plot of trajectories.} \label{fig:toy_contour_fast} \end{subfigure} \caption{\method remains stable in both fast $\beta$ regimes, as pictured here, and in slow $\beta$ regimes as in Figure \ref{fig:toy_example_method_slow}. This is unlike prior stateful methods like \texttt{Local Adam} which only offer stable convergence for fast $\beta$ values As before, we optimize the non-convex Rosenbrock function $f(x_1, x_2) = (1 - x_1)^2 + 100(x_2 - x_1^2)^2$ with $M=256$ workers and IID Gaussian noise ($\sigma=2$).} \label{fig:toy_example_method_fast} \end{figure}

\section{\texttt{LLM} Usage Declaration}
As declared in the submission form, \texttt{LLM}s were used in this work in order to aid or polish writing and for retrieval and discovery of related work. We used \texttt{GPT-5} and \texttt{Gemini 2.5 PRO} primarily to abbreviate or rephrase text or to evaluate the clarity of our writing and provide guidance on areas of improvement. We also used the deep research feature present in both models in order to discover, but not describe or interpret, additional papers for our extended literature review in \cref{app:sec:extended_related_work}. Finally, we used both models to generate plotting code and as general code assistants.

\end{document}